\documentclass{article}

\usepackage{autobit}

\title{Beyond Scalar Sensitivity: Activation-Aware Mixed-Precision LLM Quantization with Cross-Layer Refinement}

\author{
  Akihiro Yoshida \\
  Fujitsu Limited, Institute of Science Tokyo
  \And
  Yuma Ichikawa \\
  Fujitsu Limited, RIKEN Center for AIP
}

\begin{document}

\maketitle

\begin{abstract}
Mixed-precision weight quantization is commonly formulated as a Multiple-Choice Knapsack Problem (MCKP), yet existing solvers rely on scalar sensitivity proxies that collapse each weight matrix's Hessian into a single number and treat every module independently.
We prove that even the optimal scalar proxy incurs multiplicative distortion up to $\sqrt{\kappa(\mathbf{A})\kappa(\mathbf{B})}$ relative to the full activation-aware quadratic, where $\kappa(\mathbf{A})$ and $\kappa(\mathbf{B})$ denote the condition numbers of the input- and output-side Hessian factors.
This bound varies from $10^1$ to $10^{13}$ for typical LLM modules, making inter-module sensitivity ranking unreliable.
To address these limitations, we propose Cross-layer Activation-aware Sensitivity Allocation (CASA), a two-phase method.
In Stage 1, the scalar proxy is replaced by an activation-aware metric derived from the Kronecker-factored Hessian, reducing the MCKP to a form whose continuous relaxation admits a closed-form solution.
In Stage 2, a cross-layer-aware local search evaluates bit-width updates using the end-to-end model loss.
Experiments on multiple LLMs across different bit budgets show that CASA achieves lower perplexity than the latest scalar-proxy baselines, especially at ultra-low bit-widths ($<3$ bits per weight).
Moreover, the performance gain in zero-shot accuracy tracks the per-model average condition-number over modules, confirming the distortion bound as a practical indicator of scalar-proxy failure.
\end{abstract}

\section{Introduction}
\label{sec:intro}
Weight-only post-training quantization (PTQ)~\citep{lin2024awq, frantar2022gptq} has become the standard technique for deploying large language models under memory constraints.
Among the various PTQ strategies, mixed-precision quantization---where different weight matrices are quantized to different bit widths---offers a flexible trade-off between model size and accuracy~\citep{guo2025slimllm, dong2019hawq, dong2020hawq}.
The bit allocation problem is naturally formulated as a Multiple-Choice Knapsack Problem (MCKP)~\cite{lee2025qpalette, chen2021towards}: given a total bit budget, assign a bit width to each weight matrix so as to minimize the total quantization-induced loss.

The quality of any MCKP solution depends critically on the proxy that estimates per-module quantization error.
Recent mixed-precision allocators, including HIGGS~\cite{malinovskii2025higgs} and Q-Palette~\cite{lee2025qpalette}, acknowledge the role of activations but ultimately reduce this matrix-form proxy to a single scalar coefficient $\alpha_l$, yielding $\alpha_l (\nicefrac{\|\dW^{(l)}\|^{2}}{\|\mW^{(l)}\|^{2}})$.
This scalar reduction discards the directional information carried by the Hessian: the anisotropy of the activation covariance, the heterogeneous output-side sensitivities, and the channel-level structure of $\dW^{(l)}$ produced by the underlying quantizer.
The damage is not merely a loss of precision; when the MCKP solver relies on a proxy that is systematically misaligned with the true error, it allocates fewer bits to the modules that matter most.

A second, less-discussed limitation is the layer-wise independence assumption that underpins every MCKP-style allocator.
Standard formulations sum per-module errors as if the modules were statistically independent---exactly the additive structure that the knapsack objective requires.
The deep residual architecture of modern LLMs violates this assumption: quantization error injected at layer $l$ propagates through the residual stream and interacts with the error at layer $l{+}1$, producing cross-layer terms---the off-diagonal blocks of the full network Hessian---that the additive proxy cannot represent.

A natural first attempt is to fold these cross-layer terms into the MCKP objective itself, but our experiments show that this approach does not yield meaningful improvement over the self-only baseline.
These results indicate that cross-layer awareness is genuinely necessary, yet capturing it inside a single MCKP formulation is empirically difficult; it must be addressed by a mechanism outside the MCKP.

Guided by this separation of concerns, we propose Cross-layer Activation-aware Sensitivity Allocation (CASA) (see~\Cref{fig:overview}), a two-stage allocator in which Stage 1 handles the former problem inside the MCKP, and Stage 2 handles the latter problem outside it.
\textbf{Stage 1: Activation-aware MCKP.}
For each module and candidate bit-width, we evaluate the per-module
loss with the full Kronecker-factored Hessian.
Unlike scalar proxies, this cost faithfully captures activation outliers.
This advantage is not merely empirical: any scalar surrogate suffers a worst-case multiplicative distortion relative to the true quadratic (\Cref{thm:app_scalar_sharp_distortion}), which can flip the sensitivity ranking of two modules and thereby induce unbounded MCKP regret (\Cref{prop:app_misranking_mckp_regret}); in the high-rate limit our proxy further yields a water-filling allocation that no scalar surrogate can recover (\Cref{thm:app_optimal_continuous_bits}).
\textbf{Stage 2: Cross-layer local search.}
Treating the Stage-1 assignment as initialization, we perform a greedy bit-swap search whose acceptance criterion is the calibration cross-entropy loss measured after re-quantization, so that the non-additive cross-layer interactions are evaluated end-to-end rather than through any quadratic surrogate.
This stage is precisely the mechanism that operates outside the MCKP framework and recovers the cross-layer signal that no single-stage formulation captured.
The cross-aware optimum is never worse than the self-only optimum and is strictly better whenever the latter is suboptimal (\Cref{prop:app_cross_never_worse}).

Numerical experiments on multiple models (e.g., the Llama and Qwen series) across different bit budgets show that the activation-aware error proxy enhances existing scalar-based weighting methods.
The cross-layer-aware local search yields further improvement, especially at ultra-low bit-widths.
CASA successfully identifies modules with activation outliers and assigns them higher precision.
Moreover, the distortion bound $\sqrt{\kappa(\mA)\kappa(\mB)}$ established in~\Cref{thm:app_scalar_sharp_distortion} correlates with the empirical gain of CASA, validating the connection between our theory and experiments.

\begin{figure}[htbp]
  \centering
  \includegraphics[width=1.0\linewidth]{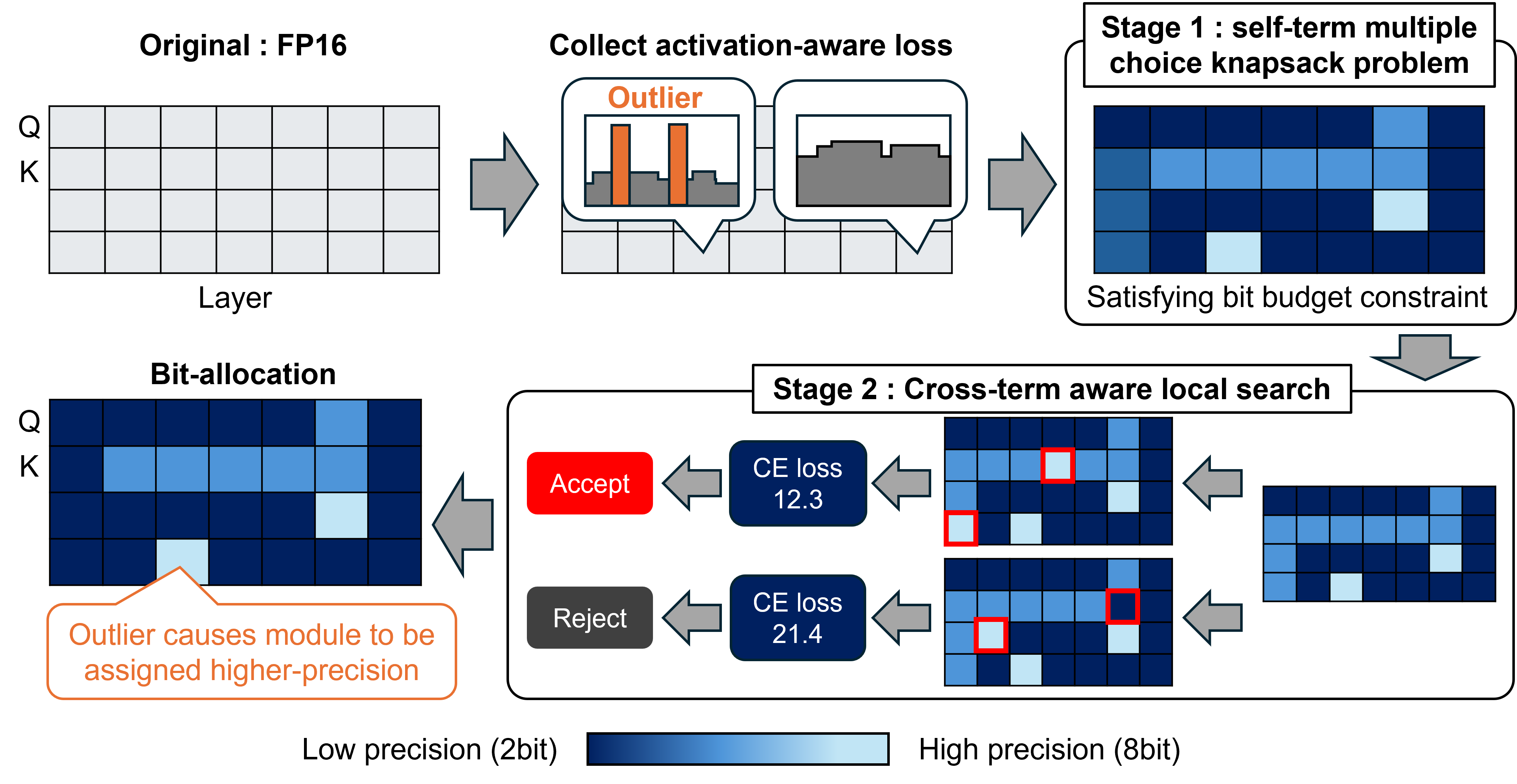}
  \caption{Overview of Cross-layer Activation-aware Sensitivity Allocation (CASA), which determines the bit allocation by accounting for activation-aware loss and cross-layer interactions. CASA efficiently detects outlier modules and assigns them higher precision.}
  \label{fig:overview}
\end{figure}

Our contributions are summarized as follows:
\begin{enumerate}
  \item We derive an activation-aware quantization-error proxy that retains the full Kronecker-factored Hessian from the second-order Taylor expansion of the task loss, and prove a worst-case distortion bound for any scalar proxy (\Cref{sec:proxy}).
  \item Building on this proxy, we propose CASA, a two-stage allocator that solves an activation-aware MCKP for self-term costs (Stage 1) and then performs a cross-entropy-driven cross-layer local search (Stage 2) that provably never worsens the Stage-1 solution (\Cref{sec:casa}).
  \item We demonstrate that CASA achieves better performance than the latest scalar-based proxy across multiple models and quantization settings.
  We further observe that CASA's improvement grows with the distortion bound, consistent with~\Cref{thm:app_scalar_sharp_distortion}~(\Cref{sec:casa_experiments}).
\end{enumerate}

\section{Preliminaries}
\label{sec:preliminaries}

\subsection{Bit-allocation as constrained optimization problem}
The mixed-precision quantization problem reduces to choosing a quantizer for each layer so as to minimize the total weighted error subject to a resource constraint~\citep{chen2021towards}.
This problem can be formulated as a multiple-choice knapsack problem (MCKP) subject to resource constraints (model size~\citep{uhlich2019mixed}, computational complexity~\citep{yang2021fracbits}, etc.).
It can be solved by a genetic algorithm~\citep{li2021brecq} or by mathematical optimization solvers~\citep{hubara2021accurate}.
In addition to the MCKP formulation, bit-allocation has been addressed via reinforcement learning~\citep{wang2019haq} or differentiable search~\citep{yang2021fracbits}.

\subsection{Sensitivity Measurement}
Many methods construct a surrogate objective to determine the bit allocation.
The HAWQ series~\citep{dong2019hawq, dong2020hawq, yao2021hawq} uses spectral information of the Hessian as a sensitivity metric.
OMPQ~\citep{ma2023ompq} prioritizes the bit allocation so that layer outputs tend to be mutually orthogonal.
Q-Palette~\citep{lee2025qpalette} measures module importance using the HIGGS metric~\citep{malinovskii2025higgs}.
These approaches fail to account for activation outliers, which are explicitly handled by recent quantization methods~\citep{lin2024awq, xiao2023smoothquant}.
More fundamentally, all of these methods reduce sensitivity to a single scalar and therefore cannot reflect activation distributions that depend on the calibration data.

\subsection{Cross-layer Awareness for Quantization}
Cross-layer awareness has been exploited for quantization itself, but rarely for bit allocation.
BRECQ~\citep{li2021brecq} iteratively refines the bit allocation by monitoring the validation loss.
QEP~\citep{arai2025quantization} propagates the previous layer's quantization loss to the next layer, thereby suppressing error accumulation.
However, cross-layer awareness for bit allocation in LLMs remains unexplored.
For vision encoders, CLADO~\citep{deng2023mixed} jointly considers both the self-term and cross-term.
InfoQ~\citep{akbulut2026infoq} studies how quantization perturbations propagate to affect the output information.
For KV-cache quantization, KVTuner~\citep{li2025kvtuner} solves a multi-objective optimization problem informed by inter-layer correlations.
Whether cross-layer interactions can actually reorder the allocation produced by a sensitivity-only integer programming formulation remains an open question.


\section{Activation-Aware Quantization Error Proxy}
\label{sec:proxy}

\subsection{Derivation}
\label{sec:proxy:derivation}

Consider a pre-trained model with $L$ layers. Let $\mW^{(l)} \in \mathbb{R}^{d_\text{out} \times d_\text{in}}$ denote the weight matrix of layer $l$, and let $\dW^{(l)} := \hat{\mW}^{(l)} - \mW^{(l)}$ be the quantization perturbation. We approximate the change in task loss $\Delta L$ through the following chain of standard assumptions:

\begin{enumerate}
  \item[(i)] Second-order Taylor expansion.
    $\Delta L \approx \frac{1}{2}\,\Delta\bm{\theta}^\top
    \mH(\bm{\theta})\,\Delta\bm{\theta}$,
    where $\mH$ is the Hessian of the loss with respect to all parameters
    $\bm{\theta}$.
  \item[(ii)] Local optimality.
    At the pre-trained weights, the gradient is approximately zero, so the first-order term in the Taylor expansion is negligible.
  \item[(iii)] Layer-wise block-diagonal Hessian.
    Off-diagonal blocks of $\mH$ between different layers are negligible, so
    $\Delta L \approx \sum_l \Delta L^{(l)}$.
  \item[(iv)] Empirical Fisher.
    Given calibration samples $\{\vx^{(n)}\}_{n=1}^{N}$, the Hessian is
    estimated via the empirical Fisher:
    $\vx^{(n)}(\vx^{(n)})^\top$ for the input side and $\mH_y^{(n)}$
    for the output side.
\end{enumerate}

Combining (i)--(iv), the per-layer loss is given by
$
  \Delta L^{(l)} \;\approx\; \frac{1}{2}\,
    \tr (\mB^{(l)}\,\dW^{(l)}\,\mA^{(l)}\,
    (\dW^{(l)})^\top),
  \label{eq:proxy_full}
$
where
$
  \mA^{(l)} = \frac{1}{N}\,\mX\mX^\top \in \mathbb{R}^{d_\text{in}\times d_\text{in}},
  ~~
  \mB^{(l)} = \frac{1}{N}\sum_{n=1}^{N}\mH_y^{(n)}
  \in \mathbb{R}^{d_\text{out}\times d_\text{out}}.
  \label{eq:AB_def}
$
Here $\mA^{(l)}$ is the input activation Gram matrix, encoding channel-wise correlations and magnitudes, and $\mB^{(l)}$ is the output-side Hessian, capturing the downstream sensitivity to perturbations in each output channel. 

\subsection{Existing Proxies as Degenerate Cases}
\label{sec:proxy:comparison}

\Cref{eq:proxy_full} provides a unified view of existing quantization error proxies as successively coarser approximations (see \Cref{tab:proxy_hierarchy}).
The scalar proxy $\alpha_l\|\dW\|^2/\|\mW\|^2$ treats all input channels and all output channels as equally important.
In practice, Transformer activations exhibit strong anisotropy: a small number of channels carry disproportionately large magnitudes (``activation outliers'').
These outlier channels make certain columns of $\dW$ far more costly than others, and analogous heterogeneity arises on the output side via $\mB$.
By collapsing $\mA$ and $\mB$ to scaled identities ($\mA \to \alpha_l\mI,\;\mB \to \mI$), scalar proxies assign equal weight to every entry of $\dW^{(l)}$, fundamentally misranking modules by their quantization sensitivity.

\begin{table}[htbp]
  \centering
  \caption{Hierarchy of quantization error proxies. All are special cases of $\tr(\mB\,\dW\,\mA\,\dW^\top)$ with different levels of approximation for $\mA$ and $\mB$.}
  \label{tab:proxy_hierarchy}
  \begin{tabular}{llll}
    \toprule
    Proxy & Approximation & Formula & Methods \\
    \midrule
    Full & $\mA^{(l)},\; \mB^{(l)}$ &
      $\tr(\mB\,\dW\,\mA\,\dW^\top)$ & CASA(\textbf{ours}) \\ \hline
    \multirow{2}{*}{Input-only} & \multirow{2}{*}{$\mB \to \mI$} &
      \multirow{2}{*}{$\frac{1}{N}\|\dW\mX\|_F^2$}
      & GPTQ~\citep{frantar2022gptq} \\
    & & & OBQ~\citep{frantar2022obq} \\ \hline
    \multirow{2}{*}{Scalar} & \multirow{2}{*}{$\mA \to \alpha_l\mI,\;\mB \to \mI$} &
      \multirow{2}{*}{$\alpha_l\|\dW\|^2 / \|\mW\|^2$}
      & HIGGS~\citep{malinovskii2025higgs} \\
    & & & Q-Palette~\citep{lee2025qpalette} \\
    \bottomrule
  \end{tabular}
\end{table}


\subsection{Theoretical Discussion}
The worst-case multiplicative distortion of any scalar proxy against the activation-aware quadratic, and the MCKP regret induced by sensitivity misranking, are stated formally in \Cref{thm:app_scalar_sharp_distortion} and~\Cref{prop:app_misranking_mckp_regret}, respectively.

\begin{restatable}[Sharp multiplicative distortion of scalar proxies]{theorem}{appscalarsharpdistortion}
\label{thm:app_scalar_sharp_distortion}
Assume $\mA\succ0$ and $\mB\succ0$. Let $\kappa(\mA):=\nicefrac{\lambda_{\max}(\mA)}{\lambda_{\min}(\mA)}$ and $\kappa(\mB):=\nicefrac{\lambda_{\max}(\mB)}{\lambda_{\min}(\mB)}$. Then
\begin{equation}
  \inf_{\alpha>0}\;\sup_{E\neq0}\max\left\{\frac{Q(E)}{Q_\alpha(E)},\frac{Q_\alpha(E)}{Q(E)}\right\} =
  \sqrt{\kappa(\mA)\kappa(\mB)} =: D.
  \label{eq:app_scalar_distortion_bound}
\end{equation}
Consequently, whenever either the input-side factor $\mA$ or the output-side factor $\mB$ is anisotropic, no scalar proxy can uniformly approximate the activation-aware quadratic proxy without incurring this worst-case multiplicative distortion.
\end{restatable}
We empirically confirms that the bound $D$ massively varies depending on the module. (See Appendix~\ref{app:distortion_bound})

\begin{proposition}[Proxy misranking can create arbitrarily large MCKP regret]
\label{prop:app_misranking_mckp_regret}
Consider two modules indexed by $i\in\{1,2\}$ and two bit-width choices, $L$ and $H$, where $L$ is of lower precision and $H$ is of higher precision. Suppose the bit budget allows exactly one module to use $H$. Let the true second-order cost be
\begin{equation}
  C_i(b)=\gamma_i d_b,
  ~~ b\in\{L,H\},
\end{equation}
where $d_L>d_H>0$ and $\gamma_i>0$. Let a proxy MCKP use the proxy cost, $\widehat C_i(b)=\widehat\gamma_i d_b$, where $\widehat\gamma_i>0$. Suppose that the true sensitivities and proxy sensitivities have opposite rankings: $\gamma_1>\gamma_2,~~ \widehat\gamma_2>\widehat\gamma_1$.
Then the true optimal allocation assigns $H$ to module $1$, whereas the proxy-optimal allocation assigns $H$ to module $2$. The true regret of the proxy-optimal allocation is
\begin{equation}
  \operatorname{Regret} =
  (\gamma_1-\gamma_2)(d_L-d_H).
\end{equation}
In particular, taking $\gamma_1=R$ and $\gamma_2=1$ makes the regret equal to $(R-1)(d_L-d_H)$, which diverges as $R\to\infty$.
\end{proposition}

\section{Cross-layer Activation-aware Sensitivity Allocation (CASA)}
\label{sec:casa}
To address the issues derived from scalar-based sensitivity and layer-wise independent treatment, we propose \textbf{Cross-layer Activation-aware Sensitivity Allocation (CASA)}.
A natural attempt is to fold both the self-term and the cross-term into a single MCKP and solve it jointly, but this approach empirically performs worse (see Appendix~\ref{app:ablation_one_stage}).
Therefore, capturing both components in a single formulation is empirically difficult, and we adopt a two-stage approach instead.
Stage 1 solves an MCKP using the activation-aware proxy of \Cref{eq:proxy_full} on the self-term alone (\Cref{sec:casa:selfterm}); Stage 2 then refines this allocation via a cross-layer local search that evaluates bit-swaps using the end-to-end loss (\Cref{sec:casa:crosslayer}); this refinement is guaranteed to produce a solution no worse than the Stage-1 one.
The overall algorithm of CASA is shown in~\Cref{alg:casa}.

\subsection{Stage 1: MCKP Formulation}
\label{sec:casa:selfterm}

Following \Cref{sec:proxy}, we formulate the bit allocation as an MCKP over all weight matrices in the model. 
Let $(l, m)$ index a specific module (e.g., $m \in \{Q, K, V, O, \text{Gate}, \text{Up}, \text{Down}\}$) in layer $l$.
Each module can be quantized at one of the candidates in $\mathcal{Q} = \{q_1, \ldots, q_{|\mathcal{Q}|}\}$ (e.g., $\{2.0, 3.0, \ldots, 6.0\}$).
We use $q \in \mathcal{Q}$ in two roles: as a generic candidate when defining per-module costs, and as the assigned bit-width $q_{l,m} \in \mathcal{Q}$ for  module $(l,m)$. 
The full allocation is denoted by $\boldsymbol{q} = (q_{l,m})_{(l,m)}$.

For each candidate $(l,m,q)$, we pre-compute the quantized weight $\hat{\mW}^{(l,m,q)}$ and the perturbation $\dW^{(l,m,q)} = \hat{\mW}^{(l,m,q)} - \mW^{(l,m)}$. 
The activation-aware proxy cost $\ell_{l,m,q}$ is:
\begin{equation}
  \ell_{l,m,q} := \tr \ \Bigl(\mB^{(l,m)}\,\dW^{(l,m,q)}\,
    \mA^{(l,m)}\,\bigl(\dW^{(l,m,q)}\bigr)^\top\Bigr).
  \label{eq:proxy_cost}
\end{equation}
Expanding in element-wise form:
$
  \ell_{l,m,q} = \sum_{i,j}\,
    b_i^{(l,m)}\cdot a_j^{(l,m)}\cdot
    (\Delta w_{ij}^{(l,m,q)})^2$,
where $a_j^{(l,m)}$ and $b_i^{(l,m)}$ are the diagonal entries of $\mA^{(l,m)}$ and $\mB^{(l,m)}$, respectively, under the diagonal approximation of $\mA$ and $\mB$.
The MCKP is:
\begin{align}
  \min_{P \in\{0,1\}} ~ & \sum_{l,m,q} P_{l,m,q}\,\ell_{l,m,q}
  \label{eq:mckp_obj} \\
  \text{subject to}~ & \sum_q P_{l,m,q} = 1
    ~ (\forall\, l,m),
    \quad
   \sum_{l,m,q} P_{l,m,q}\cdot\text{bits}_q\cdot\text{size}_{l,m} \leq B_{\rm avg} \cdot \sum_{l,m} \text{size}_{l,m}
\end{align}
where $P_{l,m,q} = 1$ indicates that module $(l,m)$ is quantized at bit width $q$, 
$\text{bits}_q$ includes the overhead from zero points and scales, 
and $B_{\rm avg}$ is the target average bits per weight (BPW).
This is an integer linear program and can be solved efficiently by existing mathematical optimization solvers.
The binary variables $P_{l,m,q}$ induce an assignment vector $\boldsymbol{q} = (q_{l,m})$ via $q_{l,m} = q \iff P_{l,m,q}=1$. Stage 2 below refines this $\boldsymbol{q}$.

In the continuous relaxation of this MCKP, where bit-widths are allowed to take real values and per-module errors follow the high-rate distortion law $\Gamma_i 2^{-2\beta_i}$ with $\Gamma_i$ proportional to $\ell_{l,m,q}$ at high $q$ and $n_i = \text{size}_{l,m}$, the optimal allocation admits a closed-form solution:

\begin{restatable}[Optimal continuous bit allocation under high-rate distortion]{theorem}{appoptimalcontinuousbits}
\label{thm:app_optimal_continuous_bits}
Let $\mathcal{I}$ be a finite nonempty index set. Assume that $\Gamma_i>0$ and $n_i>0$ for every $i\in\mathcal{I}$. Consider the continuous relaxation
\begin{equation}
  \label{eq:app_continuous_rate_problem}
  \begin{aligned}
    \min_{\{\beta_i\}_{i\in\mathcal{I}}}~
    &F(\beta):=\sum_{i\in\mathcal{I}}\Gamma_i2^{-2\beta_i} \quad
    \text{subject to}~
    &\sum_{i\in\mathcal{I}}n_i\beta_i\le B_{\rm total}.
  \end{aligned}
\end{equation}
Then the problem has a unique global minimizer. Moreover, the unique minimizer satisfies $\sum_{i\in\mathcal{I}}n_i\beta_i^{\prime}=B_{\rm total}$ and is given by
\begin{equation}
  \beta_i^{\prime} =
  \frac{1}{2}\log_2\left(\nicefrac{\Gamma_i}{n_i}\right)+c,~~~
  c =
  \frac{
    B_{\rm total}
    -
    \frac{1}{2}\sum_{j\in\mathcal{I}}n_j
    \log_2\left(\nicefrac{\Gamma_j}{n_j}\right)
  }{
    \sum_{j\in\mathcal{I}}n_j
  }.
  \label{eq:app_beta_star}
\end{equation}
Consequently, for any two modules $i,j\in\mathcal{I}$,
\begin{equation}
  \beta_i^{\prime}-\beta_j^{\prime} =
  \frac{1}{2}\log_2\left(
    \frac{\nicefrac{\Gamma_i}{n_i}}{\nicefrac{\Gamma_j}{n_j}}
  \right).
  \label{eq:app_beta_difference}
\end{equation}
Thus the optimal continuous bit width is larger for modules with larger activation-aware sensitivity per parameter, $\Gamma_i/n_i$.
\end{restatable}


\subsection{Stage 2: Local Search for considering Cross-Layer Interactions} \label{sec:casa:crosslayer}

This section presents a local search that refines the Stage-1 allocation by explicitly accounting for cross-layer interactions.

We relax Assumption (iii) in~\Cref{sec:proxy:derivation} (layer-wise block-diagonal Hessian) to a block-tridiagonal approximation. 
Specifically, we retain off-diagonal Fisher blocks only between adjacent layers $(l, l{+}1)$ and, within each such block, only the same-module-type entries $(l,m)$–$(l{+}1,m)$.
\begin{equation}
  \Delta L \;\approx\;
    \sum_{l,m}\operatorname{err}_\text{self}(l,m) + 2\sum_{l=1}^{L-1}\sum_{m}\operatorname{err}_\text{cross}(l,l\!+\!1,m).
  \label{eq:crosslayer_loss}
\end{equation}

Consideration of the cross-layer interactions is not harmful for obtaining better assignment, which is ensured by following proposition.
\begin{proposition}[Exact cross-aware optimization is never worse for the quadratic surrogate]
\label{prop:app_cross_never_worse}
Let $\mathcal{F}$ be the feasible set defined by the MCKP assignment and budget constraints. Let
\begin{equation}
  P_{\rm self}\in\operatorname*{arg\,min}_{P\in\mathcal{F}}S(P), \quad
  P_{\rm cross}\in\operatorname*{arg\,min}_{P\in\mathcal{F}}C_{\rm quad}(P).
\end{equation}
Then
\begin{equation}
  C_{\rm quad}(P_{\rm cross})
  \le
  C_{\rm quad}(P_{\rm self}).
  \label{eq:app_cross_never_worse}
\end{equation}
If $P_{\rm self}$ is not a minimizer of $C_{\rm quad}$ over $\mathcal{F}$, then the inequality is strict.
\end{proposition}
The self-terms are given by~\Cref{eq:proxy_full}. 
For the cross-terms, we apply a Kronecker factorization to the off-diagonal Fisher block:
$
  \mF^{(l,m),(l+1,m)} \approx \mathbb{E}_n[\vx_n^{(l,m)}(\vx_n^{(l+1,m)})^\top]
    \otimes
    \mathbb{E}_n [\vg_n^{(l,m)}(\vg_n^{(l+1,m)})^\top],
  \label{eq:cross_fisher}
$
For computational tractability, we approximate the Kronecker factors of \Cref{eq:cross_fisher} by their diagonals (consistent with the diagonal approximation already used for the self-term). 
The cross-term then reduces to elementwise form:
\begin{equation}
    \operatorname{err}_\text{cross}(l, l{+}1, m;\, q, q')  \;\approx\; \sum_{i,j} b_i^{\text{cross}} \cdot a_j^{\text{cross}} \cdot \Delta w_{ij}^{(l,m,q)} \cdot \Delta w_{ij}^{(l+1,m,q')}
  \label{eq:cross_err}
\end{equation}
where
$a_j^{\text{cross},(l,l+1,m)} := \mathbb{E}_n[x_{n,j}^{(l,m)} x_{n,j}^{(l+1,m)}]$, and 
$b_i^{\text{cross},(l,l+1,m)} := \mathbb{E}_n[g_{n,i}^{(l,m)}\, g_{n,i}^{(l+1,m)}]$.

The cross-layer coefficients $b_i^\text{cross}$ and $a_j^\text{cross}$ are signed, unlike their self-term counterparts, and the signs of $\Delta w_{ij}^{(l,m,q)}$ are unavailable at the planning stage 
since they depend on the candidate $q$.

Applying the triangle inequality and then the Cauchy--Schwarz bound
to~\Cref{eq:cross_err} yields the Cauchy--Schwarz cross proxy:
\begin{align}
    & \operatorname{err}_\text{cross}  (l,l+1,m;q,q')
    \le \\
    & \sum_{i,j}\sqrt{b_i^{(l,m)} b_i^{(l+1,m)}}\sqrt{a_j^{(l,m)} a_j^{(l+1,m)}}
    \left|\Delta w_{ij}^{(l,m,q)}\right| 
    \left|\Delta w_{ij}^{(l+1,m,q')}\right| 
    =:\;\bar{C}^{(l,m)}_{q,q'}.
\end{align}
Since this upper bound depends only on the self-term diagonals $a^{(l,m)}, b^{(l,m)}$ already used in \Cref{eq:proxy_full}, the table $\bigl\{\bar{C}^{(l,m)}_{q,q'}\bigr\}_{(q,q')\in\mathcal{Q}^2}$ can be precomputed without an extra calibration pass.
The total proxy objective minimized by the local search (\Cref{alg:casa}, Stage 2) is:
$
J(\boldsymbol{q}) = 2\sum_{l=1}^{L-1}\sum_{m} \bar{C}^{(l,m)}_{q_{l,m},\,q_{l+1,m}} \label{eq:proxy_total}
$.

\paragraph{Local search procedure.}
Based on the cross-term surrogate, we perform a local search to refine the MCKP solution obtained from Stage 1.
The MCKP has already minimized the self-term proxy; the local search further reduces the total objective~\eqref{eq:proxy_total} by exploiting cross-layer interactions while maintaining the budget constraint.
Each round consists of three phases: (1) candidate generation, (2) evaluation, and (3) acceptance.
In Phase 1 (candidate generation), we construct feasible swap candidates.
Each candidate consists of $k$ upgrades (modules whose bit-width index is increased by one) and $k$ downgrades (modules whose bit-width index is decreased by one), chosen so that the total effective bit budget does not exceed the target.
We write $\boldsymbol{q} \oplus s$ for the resulting allocation, in which unaffected modules retain their original bit-widths and every module satisfies $(\boldsymbol{q}\oplus s)_{l,m} \in \{1,\dots,|\mathcal{Q}|\}$.
We then rank candidates by their estimated cross-term proxy improvement.
For a swap $s$, let $q_{l,m}^{\text{new}} := (\boldsymbol{q} \oplus s)_{l,m}$ denote the bit-width assigned to module $(l,m)$ after applying $s$.
The cross-term delta is then:
\begin{equation}
  \Delta_{\text{cross}}(s) = 2\sum_{l,m}\Bigl(
        \bar{C}^{(l,m)}_{q_{l,m}^{\text{new}},\,q_{l+1,m}^{\text{new}}}
      - \bar{C}^{(l,m)}_{q_{l,m},\,q_{l+1,m}}
    \Bigr).
  \label{eq:cross_delta}
\end{equation}
In Phase 2 (evaluation), we form $\mathcal{S}$ as the top-$K$ swaps with the smallest (most negative) $\Delta_{\text{cross}}(s)$, restricted to $\Delta_{\text{cross}}(s) < 0$.
For each $s \in \mathcal{S}$, we apply the candidate allocation $\boldsymbol{q} \oplus s$ to the model and measure the next-token prediction loss on calibration data.
In Phase 3 (acceptance), we accept the candidate with the largest loss reduction:
\begin{equation}
    \boldsymbol{q}^{(r+1)} = \boldsymbol{q}^{(r)} \oplus
    \underset{s \in \mathcal{S}}{\operatorname{argmin}}\;
    \mathcal{L}_{\rm cal}(\boldsymbol{q}^{(r)} \oplus s)
    \label{eq:acceptance}
\end{equation}
where $\mathcal{L}_{\rm cal}$ is the calibration loss.
If no candidate improves upon the incumbent, the search terminates.

\section{Experiments}
\label{sec:casa_experiments}

\paragraph{Setup.}
We evaluate CASA on five models: Llama-2-7B, Llama-3-8B, Llama-3.1-8B, Qwen3-8B, and Qwen3-14B.
We compare four methods:
\begin{itemize}
    \item \textbf{Uniform}: all modules are assigned the same bit-width.
    \item \textbf{Q-Palette}~\citep{lee2025qpalette}: one of the latest mixed-precision methods, in which a scalar sensitivity parameter is computed via the HIGGS metric.
    While the original Q-Palette uses condensed quantizers, we replace them with widely-used GPTQ~\citep{frantar2022gptq}.
    \item \textbf{CASA\textsubscript{self}}: an ablation that solves only the activation-aware MCKP (Stage 1 of \Cref{alg:casa}), without the cross-term local search.
    \item \textbf{CASA} (ours): the full proposed method (\Cref{alg:casa}).
\end{itemize}
All methods use RTN (Round-to-Nearest) as the proxy quantizer. We use GPTQ~\citep{frantar2022gptq} as the main quantizer with a group size of 128 and C4 as calibration data.
The candidate bit-widths $\mathcal{Q}$ are integers from 2 to 8.
We use SCIP~\citep{bolusani2024scip} to solve the MCKP.
For the local search (Stage 2), we use $K=100$ candidate evaluations per round with a maximum of 20 rounds.
Each swap consists of one upgrade paired with one downgrade ($k=1$).
Experiments were conducted on one NVIDIA B200 GPU.
We report perplexity on WikiText-2 and average accuracy on a suite of common-sense reasoning benchmarks.

\paragraph{Results and Discussion.}
\Cref{tab:results} presents the main results.
\Cref{fig:allocation_llama} visualizes the bit allocations produced by the baselines and our proposed CASA.
More results are shown in~\Cref{tab:additional_results}.
\begin{table*}[tbhp]
\centering
\caption{WikiText-2 perplexity ($\downarrow$) and 6-task average accuracy ($\uparrow$, \%)
  for various mixed-precision quantization methods and different bit budgets.
  Each cell: PPL / ACC.
  \textbf{Bold PPL}: lowest per row; \textbf{bold ACC}: highest per row.
  $^\dagger$Uniform at 2-bit, shown for reference only.}
\label{tab:results}
\setlength{\tabcolsep}{3.5pt}
\begin{tabular}{@{}ll ccc>{\columncolor{blue!8}}c@{}}
\toprule
Model & BPW & Uniform & Q-Palette & CASA\textsubscript{self} & CASA \\
\midrule
\multirow{6}{*}{\shortstack[l]{Llama-2-7B\\{(FP16: 4.86 / 58.6)}}}
  & 2.25 & $24.20^\dagger$ / 36.8  & 11.59 / 42.4  & \textbf{11.01} / 42.1  & 11.19 / \textbf{42.7} \\
  & 2.50 & --  & 8.31 / 46.9  & 7.51 / 47.9  & \textbf{7.48} / \textbf{48.3} \\
  & 2.75 & --  & 6.93 / 51.2  & 6.22 / \textbf{52.8}  & \textbf{6.18} / 52.5 \\
  & 3.00 & 5.48 / 55.6  & 5.86 / 55.0  & \textbf{5.16} / 56.1  & 5.48 / \textbf{56.5} \\
  & 3.50 & --  & 5.24 / \textbf{57.3}  & \textbf{5.16} / 56.1  & \textbf{5.16} / 56.1 \\
  & 4.00 & 4.97 / 57.5  & 5.03 / 57.8  & \textbf{4.98} / 57.7  & \textbf{4.98} / \textbf{57.9} \\
\addlinespace
\midrule
\multirow{6}{*}{\shortstack[l]{Llama-3-8B\\{(FP16: 5.49 / 65.0)}}}
  & 2.25 & $256.28^\dagger$ / 29.8  & 47.20 / 34.0  & 35.75 / 35.6  & \textbf{22.41} / \textbf{39.2} \\
  & 2.50 & --  & 21.43 / 39.2  & 13.21 / 45.3  & \textbf{11.86} / \textbf{47.0} \\
  & 2.75 & --  & 11.95 / 49.8  & 9.69 / 52.2  & \textbf{8.83} / \textbf{53.3} \\
  & 3.00 & 12.52 / 57.7  & 9.16 / 55.8  & 7.54 / 58.5  & \textbf{7.15} / \textbf{58.6} \\
  & 3.50 & --  & 7.03 / 61.8  & \textbf{6.30} / 62.3  & \textbf{6.30} / \textbf{62.4} \\
  & 4.00 & 11.52 / 59.0  & 6.06 / 64.0  & 5.86 / 64.0  & \textbf{5.85} / \textbf{64.1} \\
\addlinespace
\midrule
\multirow{6}{*}{\shortstack[l]{Llama-3.1-8B\\{(FP16: 5.57 / 65.8)}}}
  & 2.25 & $103.52^\dagger$ / 30.5  & 40.67 / 34.3  & 25.65 / 37.6  & \textbf{20.56} / \textbf{39.6} \\
  & 2.50 & --  & 19.38 / 38.9  & 13.01 / 45.4  & \textbf{11.93} / \textbf{48.1} \\
  & 2.75 & --  & 11.82 / 50.4  & 9.56 / 53.2  & \textbf{8.78} / \textbf{54.1} \\
  & 3.00 & 19.99 / 53.9  & 9.34 / 56.5  & 7.31 / 59.0  & \textbf{7.09} / \textbf{60.4} \\
  & 3.50 & --  & 6.84 / 62.3  & 6.31 / \textbf{63.2}  & \textbf{6.30} / 63.1 \\
  & 4.00 & 9.27 / 62.2  & 6.12 / 63.9  & \textbf{5.89} / \textbf{64.6}  & \textbf{5.89} / \textbf{64.6} \\
\addlinespace
\midrule
\multirow{6}{*}{\shortstack[l]{Qwen3-8B\\{(FP16: 8.58 / 65.8)}}}
  & 2.25 & $24.32^\dagger$ / 33.9  & 20.19 / 34.3  & 14.63 / 39.7  & \textbf{13.63} / \textbf{40.8} \\
  & 2.50 & --  & 16.62 / 36.5  & 11.87 / 46.7  & \textbf{11.60} / \textbf{47.5} \\
  & 2.75 & --  & 13.89 / 38.5  & 10.60 / 51.7  & \textbf{10.29} / \textbf{54.5} \\
  & 3.00 & 9.64 / 61.1  & 12.28 / 41.7  & \textbf{9.41} / 60.2  & 9.46 / \textbf{61.5} \\
  & 3.50 & --  & 9.40 / 59.5  & 9.20 / 64.1  & \textbf{9.19} / \textbf{64.2} \\
  & 4.00 & 8.83 / 64.7  & 9.05 / 62.8  & \textbf{8.82} / \textbf{65.1}  & \textbf{8.82} / 64.7 \\
\addlinespace
\midrule
\multirow{6}{*}{\shortstack[l]{Qwen3-14B\\{(FP16: 7.58 / 69.7)}}}
  & 2.25 & $11.68^\dagger$ / 42.4  & 10.79 / 44.5  & 10.11 / 46.2  & \textbf{9.75} / \textbf{51.7} \\
  & 2.50 & --  & 10.25 / 46.0  & 9.40 / 52.0  & \textbf{9.06} / \textbf{52.6} \\
  & 2.75 & --  & 9.83 / 48.9  & \textbf{8.51} / 60.2  & \textbf{8.51} / \textbf{62.0} \\
  & 3.00 & 8.29 / 66.2  & 9.52 / 50.4  & 8.24 / 65.6  & \textbf{8.14} / \textbf{66.2} \\
  & 3.50 & --  & 8.35 / 64.6  & 8.02 / \textbf{67.7}  & \textbf{7.95} / 67.0 \\
  & 4.00 & 7.84 / 68.8  & 8.05 / 67.1  & 7.82 / \textbf{68.5}  & \textbf{7.78} / \textbf{68.5} \\
\bottomrule
\end{tabular}
\end{table*}

\begin{figure}[htbp]
  \centering
  \includegraphics[width=0.98\linewidth]{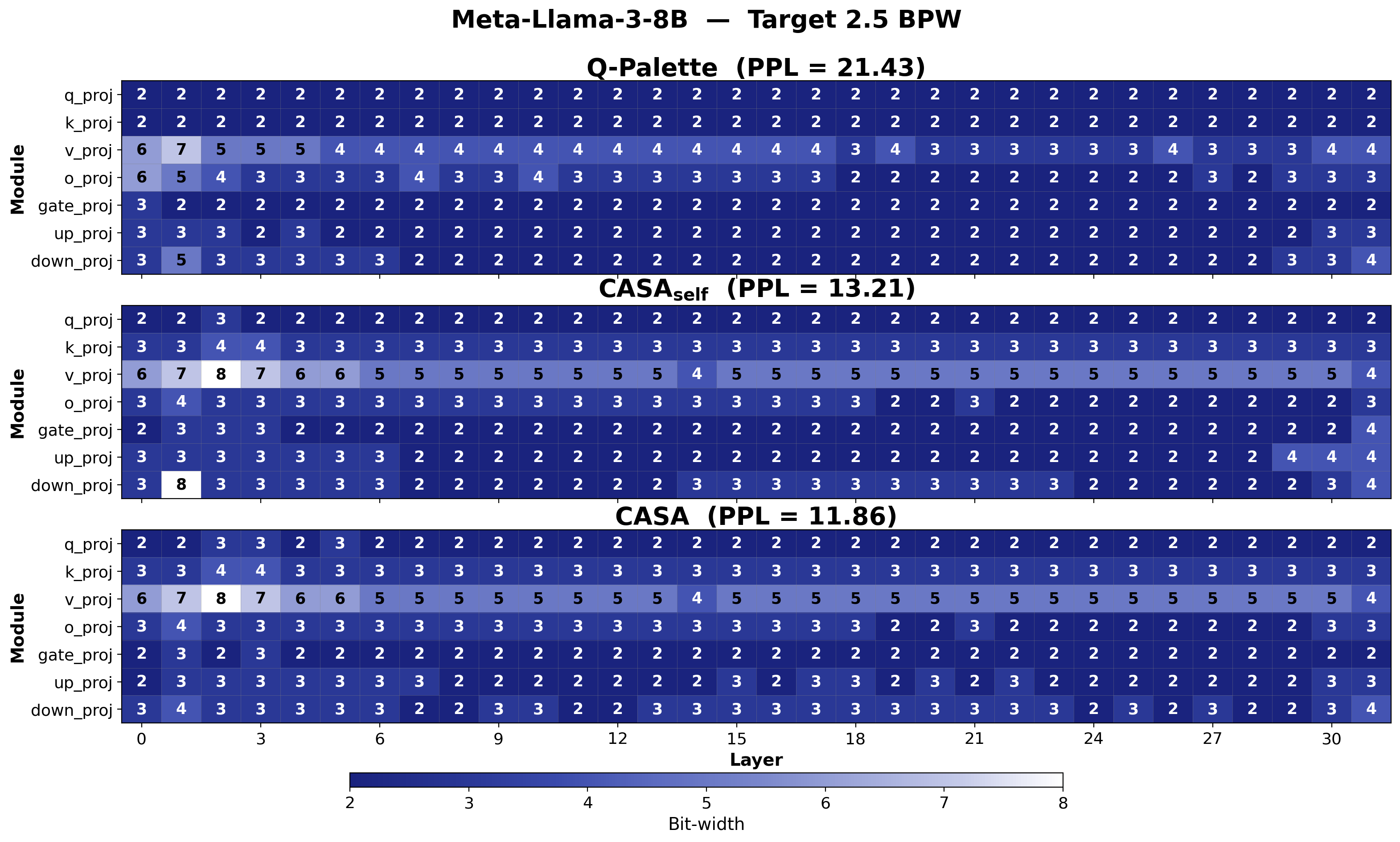}
  \caption{Bit allocation for Llama-3-8B at 2.5 BPW (blue: lower precision; white: higher precision).}
  \label{fig:allocation_llama}
\end{figure}

The key findings are:
\begin{itemize}
    \item 
    \textbf{Better results than the scalar-proxy method.}
    CASA and its ablation \textbf{CASA\textsubscript{self}} consistently achieve better performance than the latest scalar-based proxy, Q-Palette.
    \item \textbf{Massive activation channels in early-layer V projections.}
    On Llama-3-8B at 2.5 BPW, V-projection input Gram matrices $\mA$ in layers 0–4 have $\kappa(\mA)\in[2.4{\times}10^{11}, 1.5{\times}10^{15}]$. CASA flags these as critical, allocating 5–8 bits (peak 8 on layer 2). Q-Palette also rates V-projection important on average but caps at 7 bits and drops to 5 bits on layers 2–4 (vs. CASA's 8/7/6).
    \item
    \textbf{Heterogeneous output sensitivity.} The output-side Hessian $\mB$ varies substantially across output channels: in V-projection modules in particular, a subset of output dimensions carries disproportionate influence on the attention computation. CASA exploits this heterogeneity by allocating more bits to such output-sensitive modules (see \Cref{fig:allocation_llama}).
    \item
    \textbf{Cross-term aware local search improves quantized model performance.}
    In the second stage of CASA, we perform a local search that exploits the cross-term, which consistently achieves better results than CASA\textsubscript{self}.
    The improvements are largest at $2.25$--$2.75$ BPW.
    Replacing the cross-term ranking with a self-term ranking improves over CASA\textsubscript{self} but still falls short of full CASA (see Appendix~\ref{app:ablation_self_local_search}), showing that the cross-layer coupling itself is the source of the additional gain.
    \item \textbf{Distortion bound predicts the zero-shot accuracy improvement.}
    \Cref{thm:app_scalar_sharp_distortion} bounds the worst-case
    multiplicative distortion of any scalar proxy.
    \Cref{tab:distortion_vs_gap} shows the relationship between this distortion and the zero-shot accuracy improvement at 3.0~BPW: CASA's gain over Q-Palette grows substantially for models with higher distortion.
\end{itemize}

\begin{table}[t]
\centering
\caption{Mean distortion bound ($D$) and accuracy gain of CASA\textsubscript{self} over Q-Palette at 3.0 BPW (see~\Cref{thm:app_scalar_sharp_distortion}).
  $D_{\mathrm{med}}$ and $\bar{D}$ are the median and arithmetic mean over all modules and layers.}
\label{tab:distortion_vs_gap}
\begin{tabular}{@{}l rr rr r@{}}
\toprule
Model
  & $\log_{10} D_{\mathrm{med}}$
  & $\log_{10} \bar{D}$
  & \multicolumn{2}{c}{ACC (\%)}
  & $\Delta$ACC \\
\cmidrule(lr){4-5}
  &  &  & Q-Palette & CASA\textsubscript{self} & ($\%$) \\
\midrule
Llama-2-7B   & 2.81 & \textbf{7.73}  & 55.0 & 56.1 & \textbf{$+$1.1}  \\
Llama-3-8B   & 2.62 & \textbf{8.30}  & 55.8 & 58.5 & \textbf{$+$2.7}  \\
Llama-3.1-8B & 2.63 & \textbf{7.96}  & 56.5 & 59.0 & \textbf{$+$2.5}  \\
Qwen3-8B     & 4.21 & \textbf{10.86} & 41.7 & 60.2 & \textbf{$+$18.5} \\
Qwen3-14B    & 4.06 & \textbf{9.48}  & 50.4 & 65.6 & \textbf{$+$15.2} \\
\bottomrule
\end{tabular}
\end{table}

\section{Conclusion}
We proposed Cross-layer Activation-aware Sensitivity Allocation (CASA), which replaces scalar sensitivity proxies with an activation-aware quadratic derived from the Kronecker-factored Hessian and refines the allocation via cross-layer local search.
We also provided theoretical results bounding the worst-case distortion of scalar-based proxies relative to the full activation-aware quadratic.
Experiments on 7B--14B LLMs show that the activation-aware proxy consistently outperforms scalar-based proxies, with the cross-layer refinement providing further gains particularly in the ultra-low-bit regime.
The empirical improvement correlates with the per-model distortion bound $\bar{D}$ (\Cref{tab:distortion_vs_gap}), consistent with the prediction of \Cref{thm:app_scalar_sharp_distortion}.

\paragraph{Limitation.}
\label{app:limitation}
Our cross-term approximation only models adjacent layer pairs $(l, l{+}1)$ and does not capture longer-range or higher-order interactions, as faithfully incorporating them would make calibration and optimization combinatorially intractable.
Moreover, our objective is an activation-aware quadratic proxy rather than the true downstream loss, so the proxy ranking may diverge from the true perplexity ranking in extreme regimes.
\paragraph{Broader Impacts.}
On the positive side, better bit allocation enables higher-quality LLM inference on the same hardware budget, supporting on-device and edge deployment in resource-constrained settings.
On the negative side, lowering the cost of deploying capable LLMs simultaneously lowers the barrier to misuse such as disinformation generation;
this concern is common to LLM efficiency research broadly and is not amplified by our specific contribution.

\newpage
\bibliographystyle{plainnat}
\bibliography{autobit}


\newpage
\appendix

\section{Overall Algorithm}
\label{app:overall_algorithm}

We summarize the proposed CASA algorithm in \Cref{alg:casa}.
\begin{algorithm}[htbp]
    \caption{CASA: Cross-layer Activation-aware Sensitivity Allocation}
    \label{alg:casa}
    \begin{algorithmic}[1]
    \Require Pretrained model with $L$ layers and module set $\mathcal{M}$ per layer; calibration set $\mathcal{D}$; bit-width candidate set $\mathcal{Q}$; target BPW $B_{\rm avg}$; local-search rounds $R$; top-$K$ swap budget per round.
    \Ensure Per-module allocation $\boldsymbol{q}^\star = (q^\star_{l,m})$ with $q^\star_{l,m} \in \mathcal{Q}$.
    \Statex \textit{// Stage 0: shared statistics}
    \State Run one forward--backward pass on $\mathcal{D}$ to collect, for every $(l,m)$, the diagonals $a^{(l,m)}\!=\!\mathrm{diag}\,\mathbb{E}[\vx\vx^\top]$ and $b^{(l,m)}\!=\!\mathrm{diag}\,\mathbb{E}[\vg\vg^\top]$.
    \State For each $(l,m,q)\!\in\![L]\!\times\!\mathcal{M}\!\times\!\mathcal{Q}$, materialize the perturbation $\dW^{(l,m,q)}$.
    \Statex \textit{// Stage 1: self-term MCKP \hfill (\Cref{sec:casa:selfterm})}
    \State $\ell_{l,m,q}\gets\sum_{i,j} b_i^{(l,m)} a_j^{(l,m)} \bigl(\Delta w_{ij}^{(l,m,q)}\bigr)^{\!2}$
    \State $\boldsymbol{q}^{(0)}\gets\arg\min_{\boldsymbol{q}}\sum_{l,m} \ell_{l,m,q_{l,m}}\quad\text{s.t.}\quad\mathrm{BPW}(\boldsymbol{q})\le B_{\rm avg}$
    \Statex \textit{// Stage 2: cross-term local search \hfill (\Cref{sec:casa:crosslayer})}
    \State Precompute $\bar{C}^{(l,m)}_{q,q'}$ for all $q,q'\!\in\!\mathcal{Q}$, $l\!\in\![L\!-\!1]$, $m\!\in\!\mathcal{M}$
    \State $\boldsymbol{q}^\star\gets\boldsymbol{q}^{(0)}$;\quad $\mathcal{L}^\star\gets\mathcal{L}_{\rm cal}(\boldsymbol{q}^\star)$ \Comment{record calibration loss}
    \For{$r=1,\ldots,R$}
      \For{each feasible swap $s$} \Comment{proxy screening}
        \State $q_{l,m}^{\rm new}\gets(\boldsymbol{q}^\star\oplus s)_{l,m}$ for all $(l,m)$
        \State $\Delta_{\rm cross}(s)\gets 2\sum_{l,m}\!\Bigl(\bar{C}^{(l,m)}_{q_{l,m}^{\rm new},\,q_{l+1,m}^{\rm new}} - \bar{C}^{(l,m)}_{q^\star_{l,m},\,q^\star_{l+1,m}}\Bigr)$
      \EndFor
      \State $\mathcal{S}\gets$ top-$K$ swaps with $\Delta_{\rm cross}(s)<0$
      \If{$\mathcal{S}=\emptyset$} \textbf{break} \EndIf
      \State $s^\star\gets\arg\min_{s\in\mathcal{S}}\;\mathcal{L}_{\rm cal}(\boldsymbol{q}^\star\oplus s)$ \Comment{evaluate calibration loss}
      \If{$\mathcal{L}_{\rm cal}(\boldsymbol{q}^\star\oplus s^\star)\ge\mathcal{L}^\star$} \textbf{break} \EndIf
      \State $\boldsymbol{q}^\star\gets\boldsymbol{q}^\star\oplus s^\star$;\quad $\mathcal{L}^\star\gets\mathcal{L}_{\rm cal}(\boldsymbol{q}^\star)$
    \EndFor
    \State \Return $\boldsymbol{q}^\star$
    \end{algorithmic}
\end{algorithm}

\section{Per-Module Distortion Bound}
\label{app:distortion_bound}
\Cref{thm:app_scalar_sharp_distortion} shows that any scalar proxy incurs a worst-case
multiplicative distortion of $D := \sqrt{\kappa(\mA)\kappa(\mB)}$ relative to the
activation-aware quadratic proxy, where
$\kappa(\mA) = \lambda_{\max}(\mA)/\lambda_{\min}(\mA)$ and
$\kappa(\mB) = \lambda_{\max}(\mB)/\lambda_{\min}(\mB)$ are the condition
numbers of the input-side Gram matrix and the output-side curvature matrix, respectively.
When this distortion differs across modules, the scalar proxy misranks their relative sensitivities, producing suboptimal bit allocations.

To visualize the severity of this bound in practice, we compute
$\log_{10}D$ for every (layer, module) pair
in five representative models using 128 C4 calibration samples
(sequence length 256).
\Cref{fig:distortion_bound} displays the results as heatmaps in the same (module $\times$ layer) layout as the bit-allocation maps.
\paragraph{Key observations.}
\begin{itemize}
    \item The distortion bound spans many orders of magnitude within a single model.
    For example, Qwen3-8B ranges from ${\sim}10^{1.8}$ (o\_proj, layer~2)
    to ${\sim}10^{13.2}$ (gate\_proj, layer~1), a spread of over 11 orders.
    \item Early layers (0--3) consistently exhibit the highest distortion
    across all models, with Q projections in Llama-3-8B reaching $10^{10}$ and
    gate projections in Qwen3-8B reaching $10^{13}$.
    These early layers are exactly where scalar-proxy bit allocation diverges most
    from the activation-aware solution.
    \item Within each layer further weakens the scalar proxy: in Qwen3 layers, $\kappa{\sim}10^{10}$ modules coexist with $\kappa{\sim}10^{2}$ ones, so the upper bound $\sqrt{\kappa(\mA)\kappa(\mB)}$---the only a priori guarantee on scalar-proxy fidelity---varies by up to $8$ orders of magnitude across modules of the same layer, precluding any uniform ranking
    guarantee for the scalar proxy.
\end{itemize}
These heatmaps provide layer-by-layer empirical evidence that the theoretical
bound is not a loose worst-case artifact:
the distortion is large enough in practice to cause substantial misranking,
motivating the use of full activation-aware proxies as employed by CASA.

\begin{figure}[htbp]
    \centering
    \includegraphics[width=1.0\linewidth]{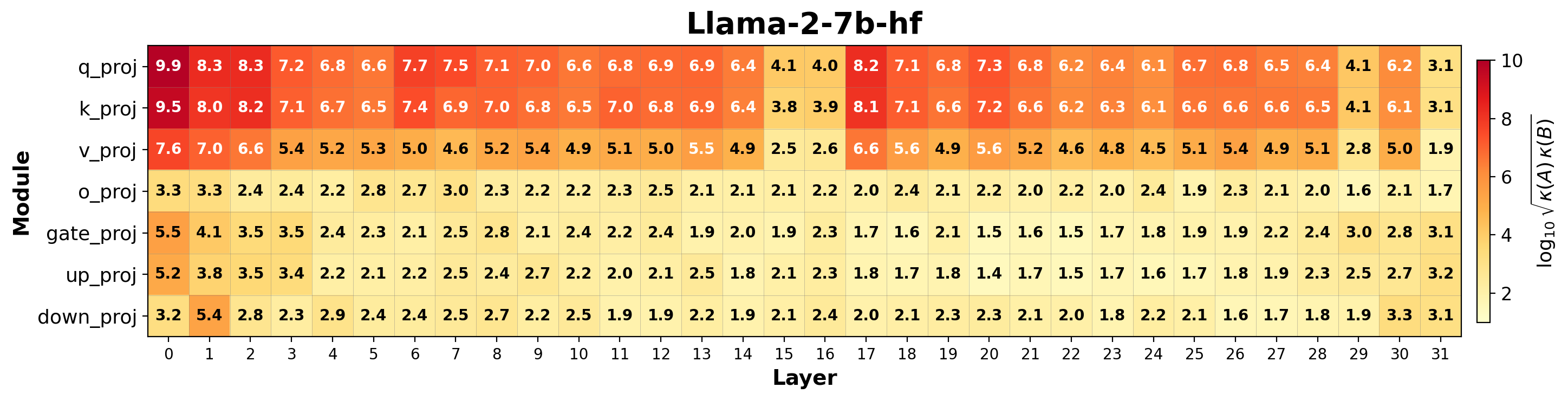}
    \includegraphics[width=1.0\linewidth]{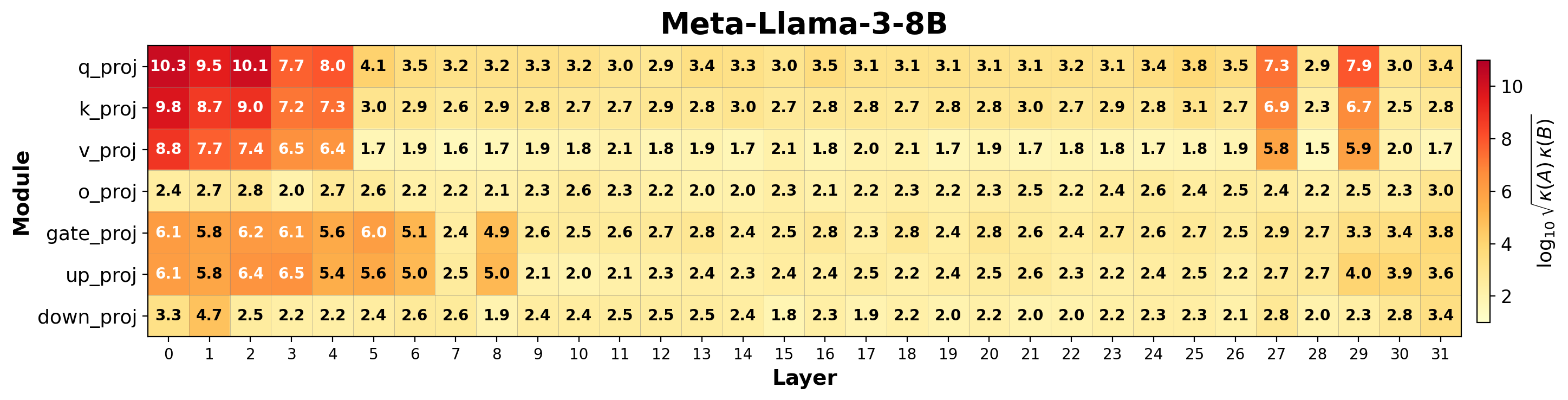}
    \includegraphics[width=1.0\linewidth]{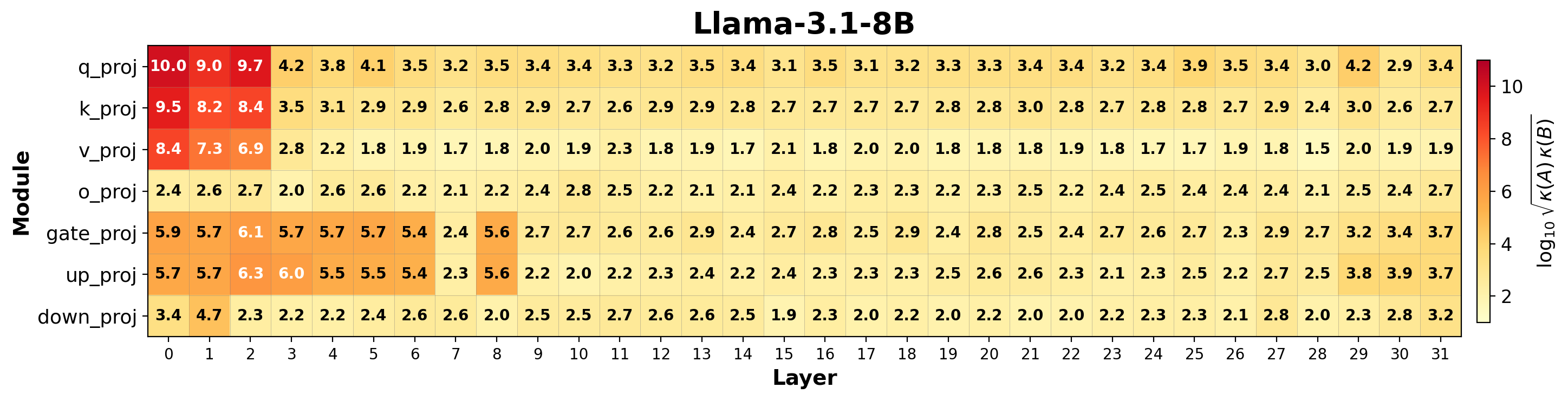}
    \includegraphics[width=1.0\linewidth]{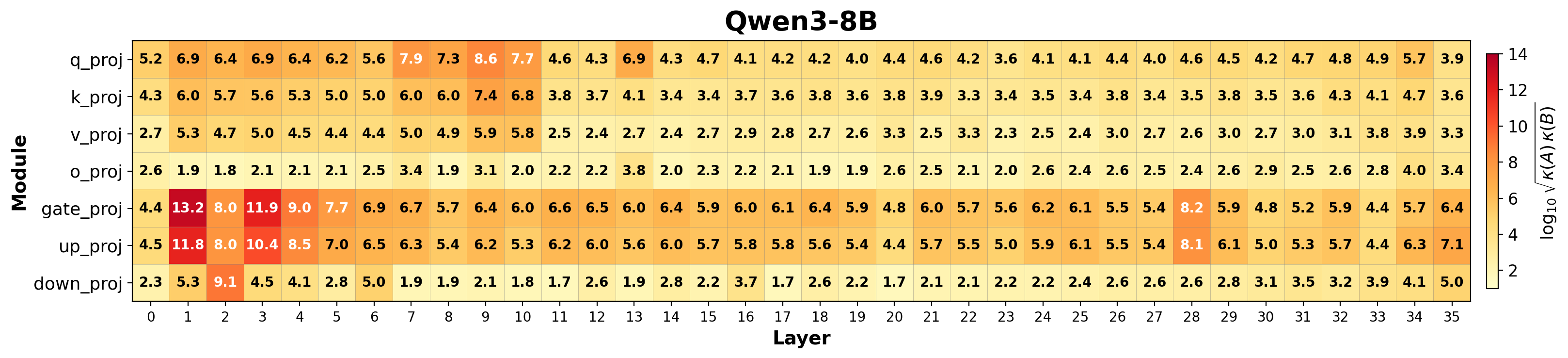}
    \includegraphics[width=1.0\linewidth]{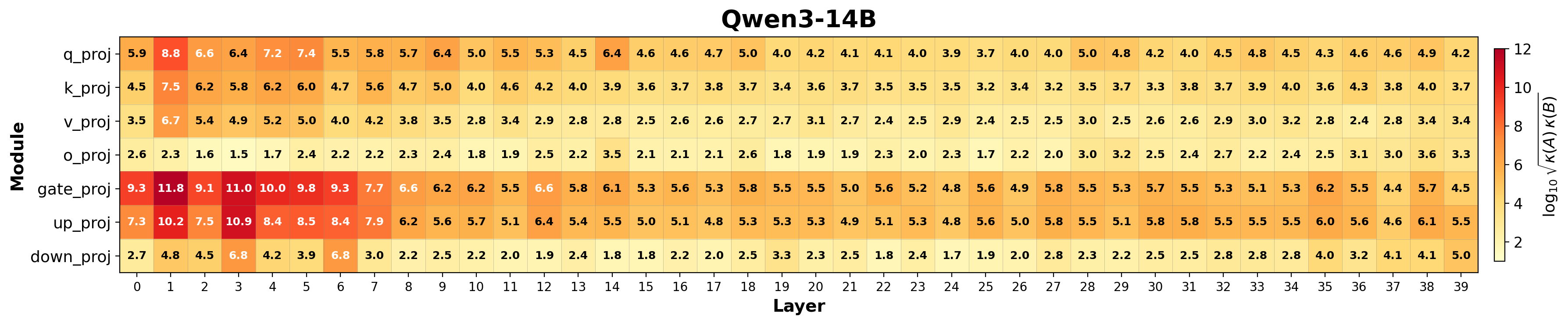}
    \caption{Per-module distortion bound $\log_{10}D$ across (layer, module) pairs for five representative models. Each cell shows $\log_{10}D$; brighter colors indicate larger worst-case scalar-proxy distortion. Early-layer attention/MLP projections exhibit the largest distortion.}
    \label{fig:distortion_bound}
\end{figure}

\section{Ablation Study: Single-Stage MCKP Combining Self- and Cross-Terms}
\label{app:ablation_one_stage}

CASA decouples the bit allocation into two stages: Stage 1 solves an MCKP on the self-term alone, and Stage 2 refines the assignment by a cross-layer-aware local search.
A natural alternative is to fold the cross-term directly into the MCKP objective and solve a single-stage problem.
We initially experimented with this approach and found that it does not improve over Stage 1 alone; this finding motivates the two-stage design of CASA.

\paragraph{Single-stage formulation.}
Adding the cross-term to~\Cref{eq:mckp_obj} yields the bilinear MCKP
\begin{equation}
  \min_{P}~
  \sum_{l,m,q} P_{l,m,q}\,\ell_{l,m,q}
  \;+\;
  2\sum_{l=1}^{L-1}\sum_{m,q,q'} P_{l,m,q}\,P_{l+1,m,q'}\,C^{(l,m)}_{q,q'},
  \label{eq:mckp_obj2}
\end{equation}
subject to the same assignment and budget constraints as in~\Cref{eq:mckp_obj}.
Here $C^{(l,m)}_{q,q'}$ is a non-negative cross-term coefficient between layers $l$ and $l{+}1$ for module $m$ when these are quantized at bit-widths $q$ and $q'$, respectively.

\paragraph{Two cross-term coefficient choices.}
We test two natural definitions of $C^{(l,m)}_{q,q'}$:
\begin{itemize}
    \item \texttt{geo\_abs}: use the Cauchy--Schwarz upper bound $\bar{C}^{(l,m)}_{q,q'}$ from~\Cref{sec:casa:crosslayer}, which depends only on the self-term diagonals already collected for Stage 1:
    \begin{equation}
      C^{(l,m),\,\mathtt{geo}}_{q,q'}
      := \bar{C}^{(l,m)}_{q,q'} = \sum_{i,j}
        \sqrt{b_i^{(l,m)} b_i^{(l+1,m)}}
        \sqrt{a_j^{(l,m)} a_j^{(l+1,m)}}
        \bigl|\Delta w_{ij}^{(l,m,q)}\bigr|
        \bigl|\Delta w_{ij}^{(l+1,m,q')}\bigr|.
      \label{eq:cross_geo_abs}
    \end{equation}
    \item \texttt{raw\_abs}: use the absolute values of the cross-layer Fisher diagonals defined in \Cref{eq:cross_err},
    \begin{equation}
      C^{(l,m),\,\mathtt{raw}}_{q,q'}
      := \sum_{i,j}
        \bigl|b_i^{\text{cross},(l,l+1,m)}\bigr|
        \bigl|a_j^{\text{cross},(l,l+1,m)}\bigr|
        \bigl|\Delta w_{ij}^{(l,m,q)}\bigr|
        \bigl|\Delta w_{ij}^{(l+1,m,q')}\bigr|.
      \label{eq:cross_raw_abs}
    \end{equation}
    Unlike \texttt{geo\_abs}, which reuses the self-term diagonals from Stage 1, \texttt{raw\_abs} requires the cross-Fisher diagonals $a^{\text{cross}}, b^{\text{cross}}$. These can be collected in the same calibration pass as the self-term diagonals, at the cost of additional memory for storing adjacent-layer activations.
\end{itemize}

\paragraph{McCormick linearization.}
Equation~\eqref{eq:mckp_obj2} is bilinear in $P$.
Introducing auxiliary binary variables $z_{l,m,q,q'} \in \{0,1\}$ together with the standard McCormick envelope linearizes the product into an integer linear programming problem:
\begin{align}
  \min_{P,z}~ &
    \sum_{l,m,q} P_{l,m,q}\,\ell_{l,m,q}
    \;+\;
    2\sum_{l=1}^{L-1}\sum_{m,q,q'} z_{l,m,q,q'}\,C^{(l,m)}_{q,q'},
  \label{eq:lin_qmckp_obj} \\
  \text{s.t.}~ &
    z_{l,m,q,q'} \le P_{l,m,q},\;\;
    z_{l,m,q,q'} \le P_{l+1,m,q'},\;\;
    z_{l,m,q,q'} \ge P_{l,m,q} + P_{l+1,m,q'} - 1,
  \label{eq:lin_qmckp_mc}
\end{align}
together with the assignment and budget constraints of~\Cref{eq:mckp_obj}.

\paragraph{Results.}
We use the same experimental setup as in~\Cref{sec:casa_experiments}; \Cref{tab:cross_term_single_stage} reports the results.
The linearized single-stage MCKP fails to improve over the self-term-only baseline under both coefficient choices, although for different reasons:
\begin{itemize}
  \item Under \texttt{geo\_abs}, the LP relaxation of the McCormick envelope is loose: the solver collapses onto degenerate, near-uniform allocations (marked $^\dagger$), and the resulting perplexity is not even monotone in the bit budget.
  \item Under \texttt{raw\_abs}, the degeneracy is avoided, but the allocations only match or marginally improve the self-term-only MCKP despite a substantially larger optimization problem.
\end{itemize}
These observations confirm that cross-layer effects are better handled \emph{outside} the MCKP formulation, motivating the two-stage design of CASA.

\begin{table*}[htbp]
\centering
\caption{Wikitext-2 perplexity ($\downarrow$) and 4-task average accuracy ($\uparrow$, \%)
    for single-stage MCKP with different objective functions.
    $\ell_{\rm self}$: self-term only;
    $\ell_{\rm self} + C^{\rm geo}$: adds the geometric-mean cross-term aggregate (\Cref{eq:cross_geo_abs});
    $\ell_{\rm self} + C^{\rm raw}$: adds the raw-absolute cross-term aggregate (\Cref{eq:cross_raw_abs}).
    All use GPTQ with group size 128, C4 calibration, and no local search.
    Each cell: PPL / ACC.
    \textbf{Bold PPL}: lowest per row; \textbf{bold ACC}: highest per row.
    $^\dagger$Solver collapse: identical allocation across all BPW targets due to McCormick LP-relaxation looseness.
  }
\label{tab:cross_term_single_stage}
\setlength{\tabcolsep}{5pt}
\begin{tabular}{@{}ll ccc@{}}
\toprule
Model & BPW & $\ell_{\mathrm{self}}$ & $\ell_{\mathrm{self}} +  C^{\rm geo}$ & $\ell_{\mathrm{self}} +  C^{\rm raw}$ \\
\midrule
\multirow{8}{*}{Llama-2-7B}
  & 2.25 & 11.01 / \textbf{43.0}  & 18.94$^\dagger$ / 39.8  & \textbf{10.93} / 42.8 \\
  & 2.50 & \textbf{7.51} / \textbf{47.8}  & 18.94$^\dagger$ / 39.8  & 7.61 / 46.9 \\
  & 2.75 & 6.22 / \textbf{52.1}  & 18.61$^\dagger$ / 40.2  & \textbf{6.22} / 51.6 \\
  & 3.00 & \textbf{5.16} / \textbf{55.3}  & 18.94$^\dagger$ / 39.8  & 5.53 / 54.9 \\
  & 3.25 & \textbf{5.28} / \textbf{56.4}  & 17.34$^\dagger$ / 39.1  & \textbf{5.28} / 56.0 \\
  & 3.50 & \textbf{5.16} / 55.3  & 17.52$^\dagger$ / 38.8  & \textbf{5.16} / \textbf{55.4} \\
  & 3.75 & \textbf{5.05} / \textbf{56.5}  & 19.19$^\dagger$ / 38.9  & \textbf{5.05} / \textbf{56.5} \\
  & 4.00 & \textbf{4.98} / 57.1  & 17.52$^\dagger$ / 39.3  & \textbf{4.98} / \textbf{57.2} \\
\addlinespace
\midrule
\multirow{8}{*}{Llama-3-8B}
  & 2.25 & 35.75 / 36.8  & \textbf{27.05} / \textbf{37.8}  & 35.62 / 37.0 \\
  & 2.50 & \textbf{13.21} / \textbf{46.2}  & 13.21 / 45.4  & 13.49 / 44.5 \\
  & 2.75 & 9.69 / 53.0  & 117.30$^\dagger$ / 33.1  & \textbf{9.67} / \textbf{53.7} \\
  & 3.00 & 7.54 / \textbf{60.3}  & 14.60$^\dagger$ / 58.9  & \textbf{7.49} / 59.6 \\
  & 3.25 & 6.60 / \textbf{62.8}  & 14.60$^\dagger$ / 58.9  & \textbf{6.60} / 62.7 \\
  & 3.50 & \textbf{6.29} / \textbf{64.2}  & 14.60$^\dagger$ / 58.9  & 6.30 / \textbf{64.2} \\
  & 3.75 & 6.07 / \textbf{65.2}  & 14.60$^\dagger$ / 58.9  & \textbf{6.07} / 65.1 \\
  & 4.00 & \textbf{5.86} / \textbf{65.8}  & 14.60$^\dagger$ / 58.9  & 5.87 / \textbf{65.8} \\
\addlinespace
\midrule
\multirow{8}{*}{Llama-3.1-8B}
  & 2.25 & 25.65 / 38.4  & 26.32$^\dagger$ / 37.4  & \textbf{25.59} / \textbf{38.5} \\
  & 2.50 & 13.01 / \textbf{46.0}  & 78.15$^\dagger$ / 32.6  & \textbf{12.82} / 45.8 \\
  & 2.75 & 9.56 / \textbf{54.0}  & 79.01$^\dagger$ / 31.5  & \textbf{9.42} / \textbf{54.0} \\
  & 3.00 & \textbf{7.31} / 59.8  & 22.70$^\dagger$ / 55.2  & 7.34 / \textbf{60.2} \\
  & 3.25 & 6.57 / \textbf{63.6}  & 22.70$^\dagger$ / 55.2  & \textbf{6.56} / 63.5 \\
  & 3.50 & 6.31 / \textbf{64.7}  & 22.70$^\dagger$ / 55.2  & \textbf{6.30} / 64.3 \\
  & 3.75 & \textbf{6.09} / \textbf{65.3}  & 22.70$^\dagger$ / 55.2  & \textbf{6.09} / 64.9 \\
  & 4.00 & 5.89 / \textbf{65.8}  & 22.70$^\dagger$ / 55.2  & \textbf{5.88} / 65.4 \\
\addlinespace
\midrule
\multirow{8}{*}{Qwen3-8B}
  & 2.25 & \textbf{14.63} / 40.5  & 24.98$^\dagger$ / 34.4  & 14.97 / \textbf{40.6} \\
  & 2.50 & 11.87 / 47.7  & 26.31$^\dagger$ / 34.0  & \textbf{11.60} / \textbf{49.3} \\
  & 2.75 & \textbf{10.60} / 54.1  & 23.44$^\dagger$ / 34.3  & \textbf{10.60} / \textbf{54.3} \\
  & 3.00 & 9.41 / 62.6  & 9.60$^\dagger$ / \textbf{63.5}  & \textbf{9.37} / 62.8 \\
  & 3.25 & \textbf{9.25} / 64.9  & 9.60$^\dagger$ / 63.5  & \textbf{9.25} / \textbf{65.4} \\
  & 3.50 & 9.20 / \textbf{66.3}  & 9.60$^\dagger$ / 63.5  & \textbf{9.17} / 66.2 \\
  & 3.75 & 8.94 / \textbf{66.7}  & 9.60$^\dagger$ / 63.5  & \textbf{8.88} / 66.1 \\
  & 4.00 & \textbf{8.82} / \textbf{67.3}  & 9.60$^\dagger$ / 63.5  & \textbf{8.82} / \textbf{67.3} \\
\addlinespace
\midrule
\multirow{8}{*}{Qwen3-14B}
  & 2.25 & \textbf{10.11} / 47.9  & 11.85$^\dagger$ / 44.0  & \textbf{10.11} / \textbf{49.6} \\
  & 2.50 & 9.40 / 53.6  & 11.76$^\dagger$ / 43.3  & \textbf{9.36} / \textbf{55.2} \\
  & 2.75 & \textbf{8.51} / \textbf{62.4}  & 11.77$^\dagger$ / 43.3  & 8.57 / 61.4 \\
  & 3.00 & 8.24 / 68.0  & 8.25$^\dagger$ / \textbf{68.8}  & \textbf{8.22} / 68.3 \\
  & 3.25 & 8.12 / \textbf{69.2}  & 8.25$^\dagger$ / 68.8  & \textbf{8.11} / \textbf{69.2} \\
  & 3.50 & 8.02 / \textbf{70.0}  & 8.25$^\dagger$ / 68.8  & \textbf{7.94} / 69.8 \\
  & 3.75 & 7.87 / \textbf{70.6}  & 8.25$^\dagger$ / 68.8  & \textbf{7.82} / 70.5 \\
  & 4.00 & \textbf{7.82} / 71.1  & 8.25$^\dagger$ / 68.8  & 7.83 / \textbf{71.2} \\
\bottomrule
\end{tabular}
\end{table*}

\newpage
\section{Ablation Study: Local Search without Cross-Layer Awareness}
\label{app:ablation_self_local_search}

In the second stage of CASA, we perform a local search that evaluates candidates using the full proxy objective including cross-term coupling between layers.
To isolate the contribution of this cross-term awareness, we compare three configurations in \Cref{tab:ablation_lsself}:
(1) CASA\textsubscript{self}, which uses the Stage-1 MCKP solution without any local search;
(2) +LS\textsubscript{self}, which applies local search with the same budget as CASA but ranks candidates by the self-term proxy alone; and
(3) the full CASA, whose local search ranks candidates by the complete proxy including cross-terms.
+LS\textsubscript{self} almost always improves over CASA\textsubscript{self} (14 of 15 configurations), confirming that local search itself is beneficial; the full CASA further improves over +LS\textsubscript{self} in 13 of 15 configurations.
The gap is most pronounced in ultra-low-bit regimes on the Llama-3 family: on Llama-3-8B at 2.25 BPW, perplexity drops from 28.15 (+LS\textsubscript{self}) to 22.41 (CASA), a 20\% relative reduction.
These results demonstrate that the self-term proxy alone cannot capture quantization error interactions across layers, and that incorporating cross-term coupling into the search objective is particularly beneficial for bit allocation in the low-bit regime.

\begin{table*}[htbp]
\centering
\caption{Ablation: self-only local search (+LS\textsubscript{self}) versus CASA\textsubscript{self} (no local search) and CASA (cross-term local search) in the low-bit regime.
  All methods share the same Stage 1 MCKP (self-only proxy) and GPTQ quantizer (group size 128, C4 calibration).
  +LS\textsubscript{self} uses the same search budget as CASA but ranks candidates by self-term proxy only (no cross-term coupling).
  Each cell: PPL / ACC.
  \textbf{Bold}: best per row.}
\label{tab:ablation_lsself}
\begin{tabular}{@{}ll cc>{\columncolor{blue!8}}c@{}}
\toprule
Model & BPW
  & CASA\textsubscript{self}
  & +LS\textsubscript{self}
  & CASA \\
\midrule
\multirow{3}{*}{\shortstack[l]{Llama-2-7B\\{(FP16: 4.86 / 58.6)}}}
  & 2.25 & 11.01 / 42.1  & \textbf{10.51} / 42.4  & 11.19 / \textbf{42.7} \\
  & 2.50 & 7.51 / 47.9  & \textbf{7.48} / 47.9  & \textbf{7.48} / \textbf{48.3} \\
  & 2.75 & 6.22 / \textbf{52.8}  & 6.21 / 52.2  & \textbf{6.18} / 52.5 \\
\midrule
\multirow{3}{*}{\shortstack[l]{Llama-3-8B\\{(FP16: 5.49 / 65.0)}}}
  & 2.25 & 35.75 / 35.6  & 28.15 / 38.0  & \textbf{22.41} / \textbf{39.2} \\
  & 2.50 & 13.21 / 45.3  & 12.44 / 46.6  & \textbf{11.86} / \textbf{47.0} \\
  & 2.75 & 9.69 / 52.2   & 9.47 / 52.9   & \textbf{8.83} / \textbf{53.3} \\
\midrule
\multirow{3}{*}{\shortstack[l]{Llama-3.1-8B\\{(FP16: 5.57 / 65.8)}}}
  & 2.25 & 25.65 / 37.6  & 23.24 / 38.7  & \textbf{20.56} / \textbf{39.6} \\
  & 2.50 & 13.01 / 45.4  & 12.38 / 45.7  & \textbf{11.93} / \textbf{48.1} \\
  & 2.75 & 9.56 / 53.2   & 9.29 / 53.6   & \textbf{8.78} / \textbf{54.1} \\
\midrule
\multirow{3}{*}{\shortstack[l]{Qwen3-8B\\{(FP16: 8.58 / 65.8)}}}
  & 2.25 & 14.63 / 39.7  & 14.10 / 40.2  & \textbf{13.63} / \textbf{40.8} \\
  & 2.50 & 11.87 / 46.7  & 11.68 / 46.7  & \textbf{11.60} / \textbf{47.5} \\
  & 2.75 & 10.60 / 51.7  & 10.52 / 53.1  & \textbf{10.29} / \textbf{54.5} \\
\addlinespace
\midrule
\multirow{3}{*}{\shortstack[l]{Qwen3-14B\\{(FP16: 7.58 / 69.7)}}}
  & 2.25 & 10.11 / 46.2  & \textbf{9.63} / 49.9   & 9.75 / \textbf{51.7} \\
  & 2.50 & 9.40 / 52.0   & 9.18 / \textbf{54.2}   & \textbf{9.06} / 52.6 \\
  & 2.75 & \textbf{8.51} / 60.2  & 8.57 / 60.2  & \textbf{8.51} / \textbf{62.0} \\
\bottomrule
\end{tabular}
\end{table*}

\newpage
\section{Visualization of bit-allocation}
\label{app:visualization}
In addition to~\Cref{fig:allocation_llama}, we present bit-allocation results for additional models in Figures~\ref{fig:allocation_llama2_7b}–-\ref{fig:allocation_qwen14b}.

\begin{figure}[htbp]
  \centering
  \includegraphics[width=0.95\linewidth]{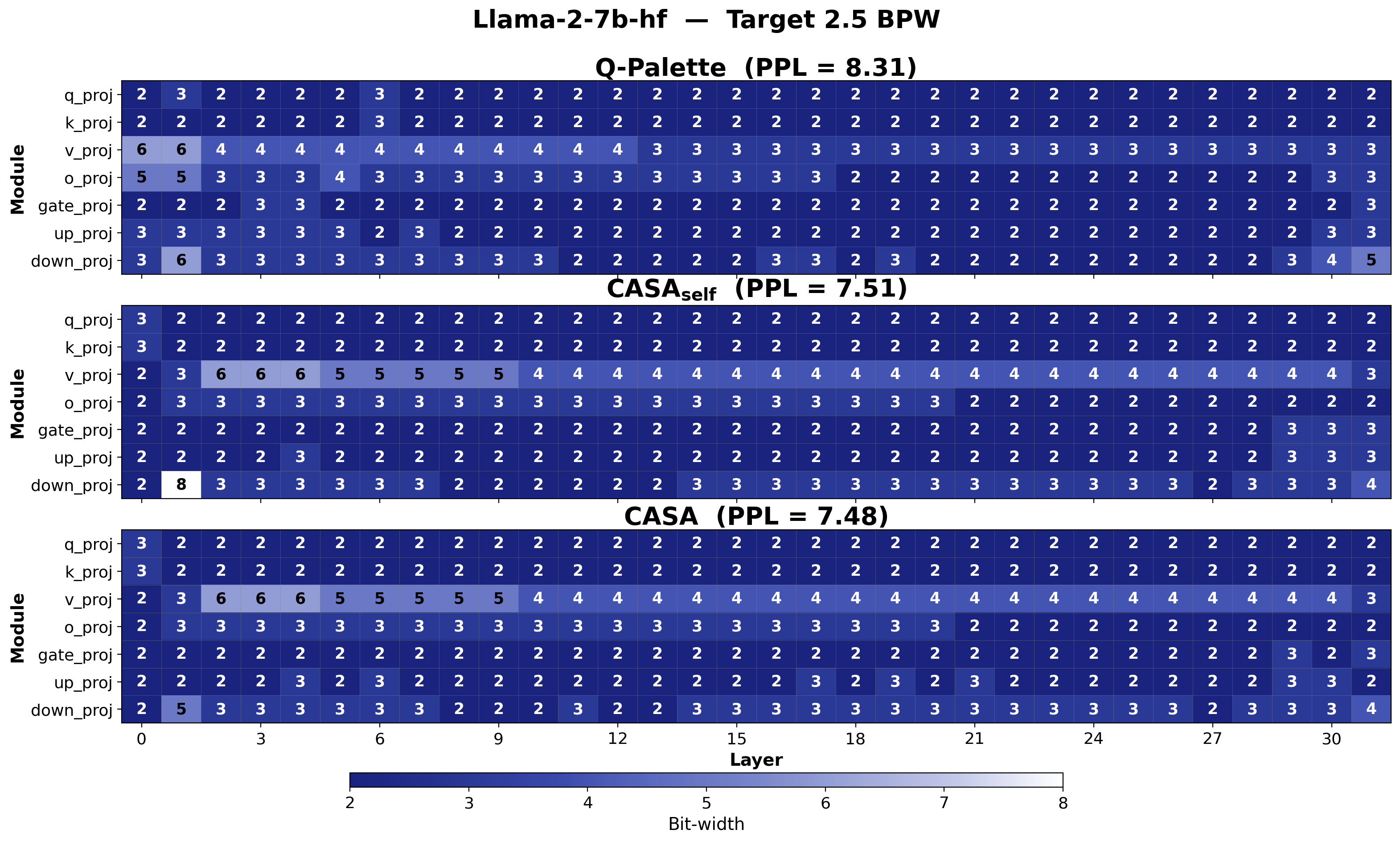}
  \caption{Visualization of bit-allocation for Llama-2-7B at 2.5BPW (blue: lower precision; white: higher precision)}
  \label{fig:allocation_llama2_7b}
\end{figure}

\begin{figure}[htbp]
  \centering
  \includegraphics[width=0.95\linewidth]{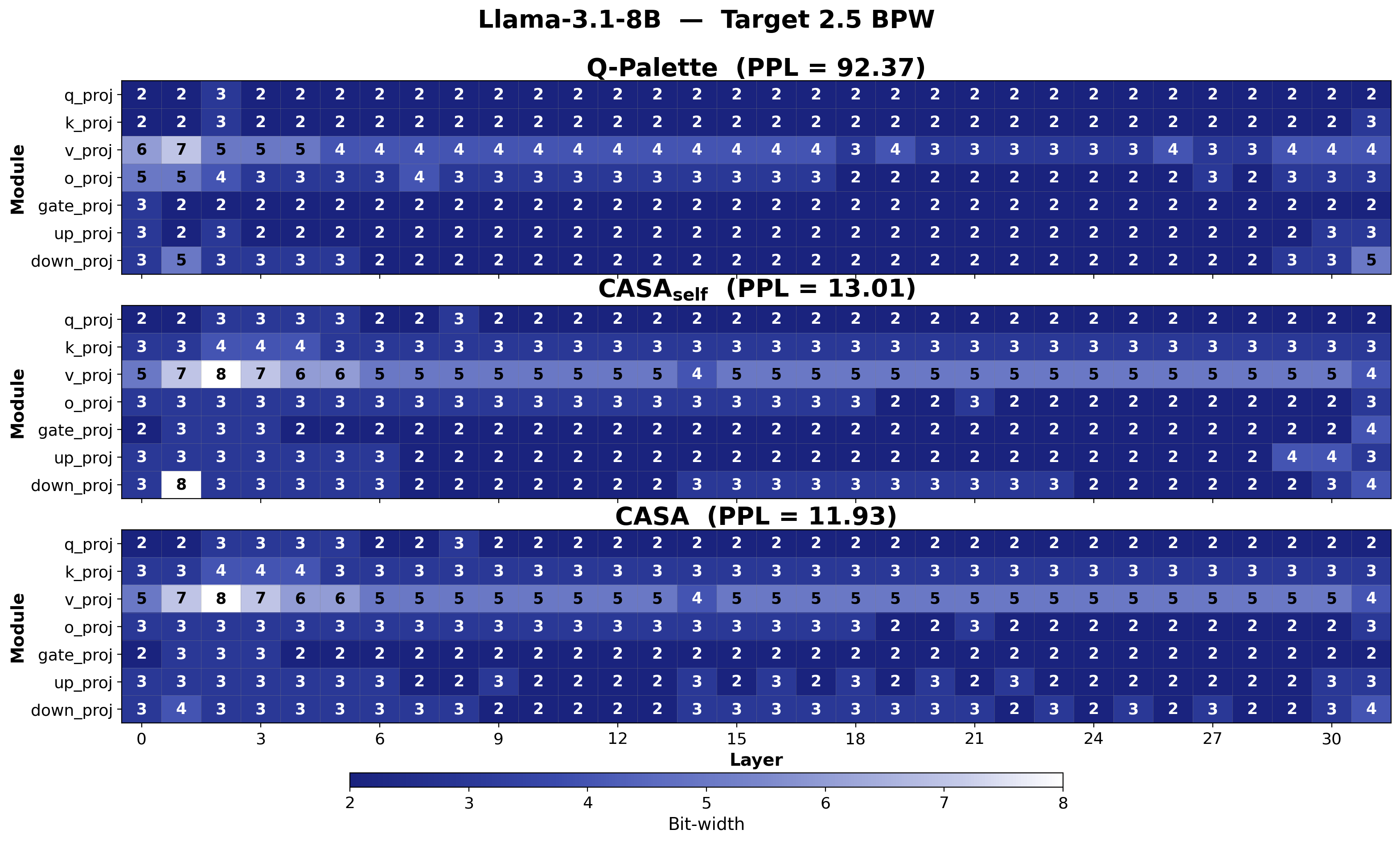}
  \caption{Visualization of bit-allocation for Llama-3.1-8B at 2.5BPW (blue: lower precision; white: higher precision)}
  \label{fig:allocation_llama31_8b}
\end{figure}

\begin{figure}[htbp]
  \centering
  \includegraphics[width=0.95\linewidth]{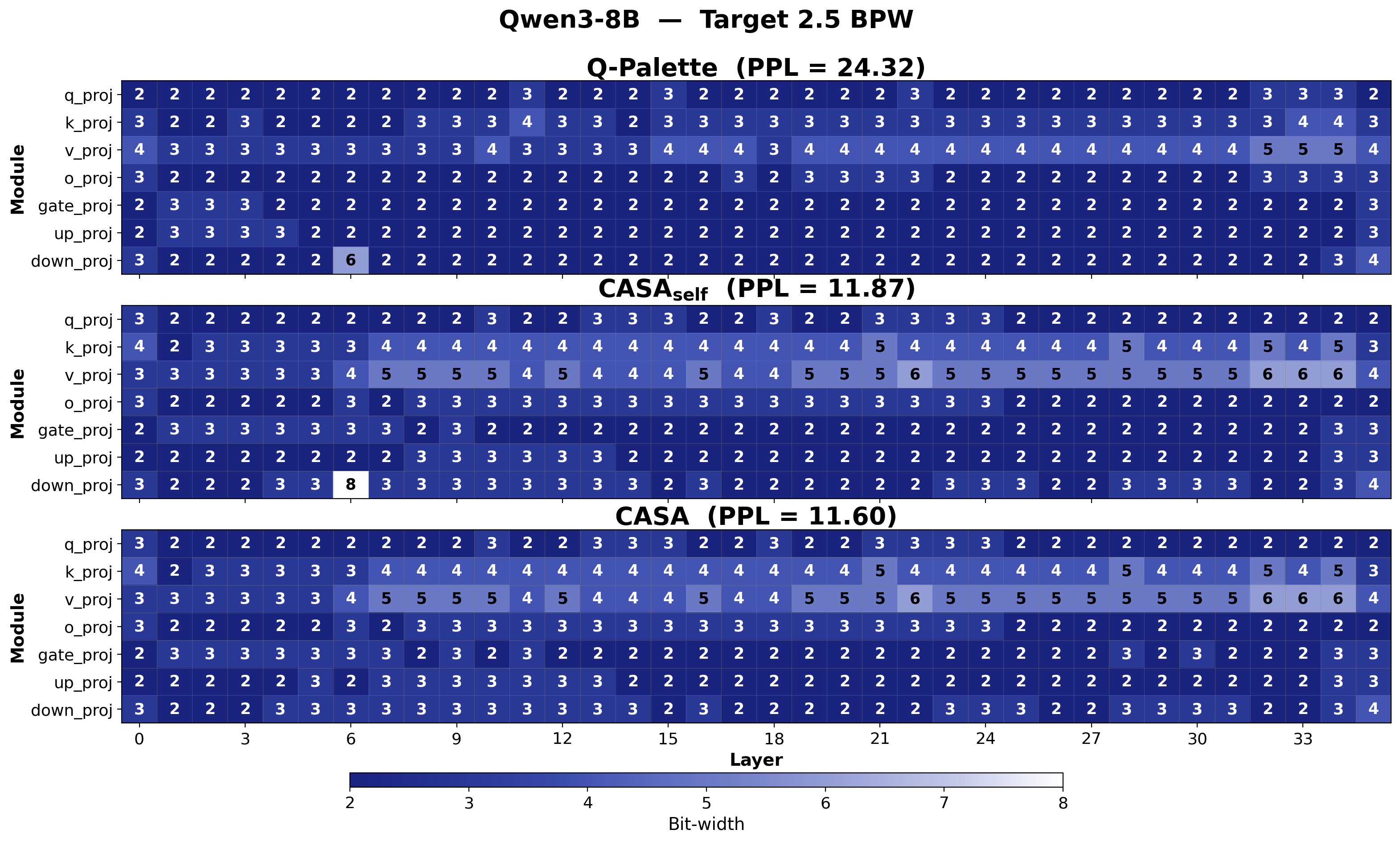}
  \caption{Visualization of bit-allocation for Qwen3-8B at 2.5BPW (blue: lower precision; white: higher precision)}
  \label{fig:allocation_qwen8b}
\end{figure}

\begin{figure}[htbp]
  \centering
  \includegraphics[width=0.95\linewidth]{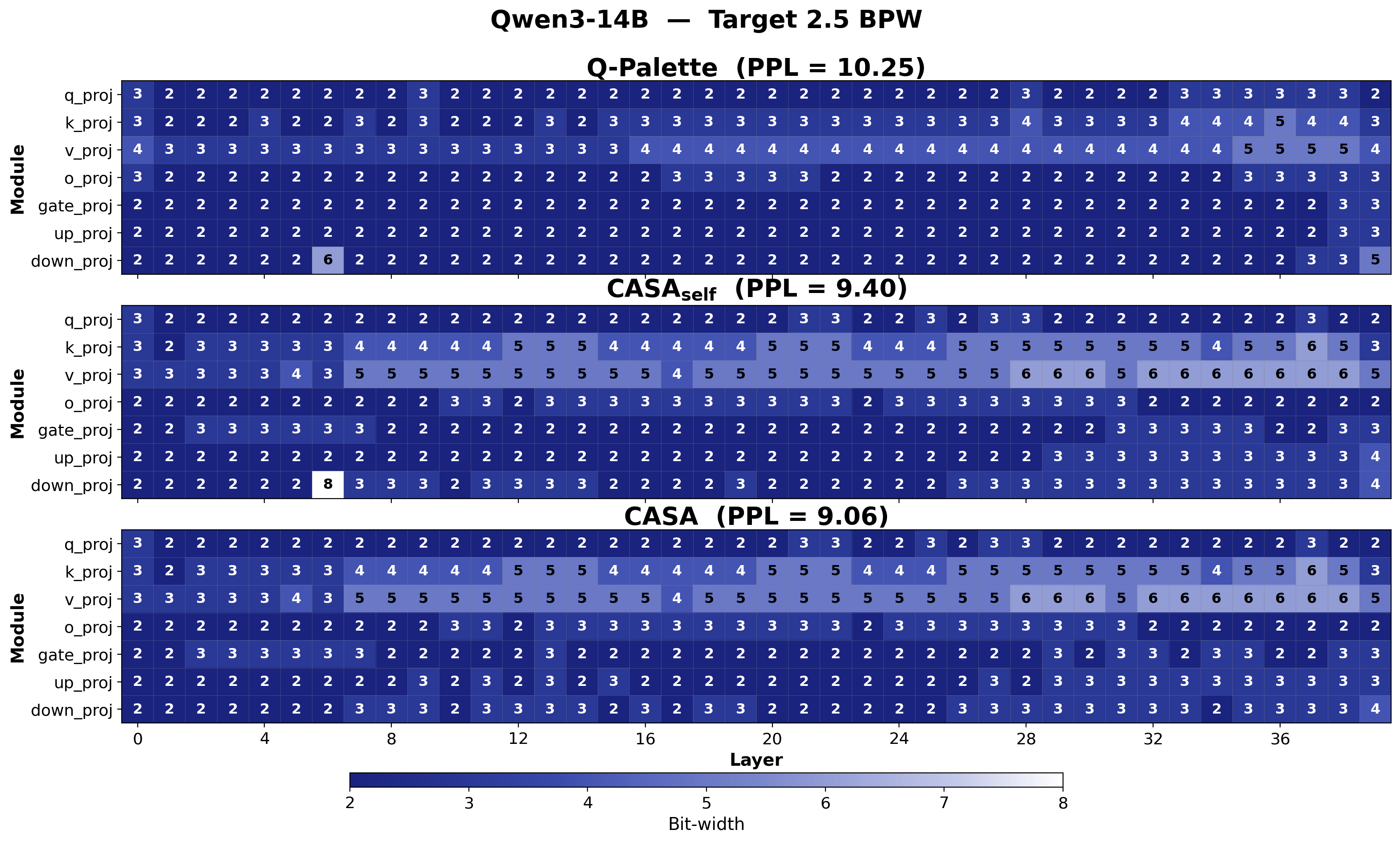}
  \caption{Visualization of bit-allocation for Qwen3-14B at 2.5BPW (blue: lower precision; white: higher precision)}
  \label{fig:allocation_qwen14b}
\end{figure}


\newpage
\section{Additional Experimental Results}
In addition to ~\Cref{tab:results},~\Cref{tab:additional_results} shows the additional experimental results under different bit budgets.
\begin{table*}[tbhp]
\centering
\caption{WikiText-2 perplexity ($\downarrow$) and 6-task average accuracy ($\uparrow$, \%)
  for various mixed-precision quantization methods and different bit budgets.
  Each cell: PPL / ACC.
  \textbf{Bold PPL}: lowest per row; \textbf{bold ACC}: highest per row
  (among mixed-precision methods).
  $^\dagger$Uniform at 2-bit as reference values.}
\label{tab:additional_results}
\setlength{\tabcolsep}{3.5pt}
\begin{tabular}{@{}l ccc>{\columncolor{blue!8}}c@{}}
\toprule
Model & BPW &  Q-Palette & CASA\textsubscript{self} & CASA \\
\midrule
\multirow{2}{*}{\shortstack[l]{Llama-2-7B\\{(FP16: 4.86 / 58.6)}}}
  & 3.25 &  5.37 / 56.1  & \textbf{5.28} / 56.8  & \textbf{5.28} / \textbf{56.9} \\
  & 3.75 &  5.14 / \textbf{57.8}  & \textbf{5.05} / 57.3  & \textbf{5.05} / 56.9 \\
\midrule
\multirow{2}{*}{\shortstack[l]{Llama-3-8B\\{(FP16: 5.49 / 65.0)}}}
  & 3.25 &  7.71 / 59.8  & 6.60 / 61.1  & \textbf{6.59} / \textbf{61.2} \\
  & 3.75 &  6.29 / \textbf{63.4}  & \textbf{6.07} / 63.2  & \textbf{6.07} / 63.2 \\
\midrule
\multirow{2}{*}{\shortstack[l]{Llama-3.1-8B\\{(FP16: 5.57 / 65.8)}}}
  & 3.25 &  7.61 / 60.2  & 6.57 / \textbf{62.1}  & \textbf{6.55} / 61.7 \\
  & 3.75 &  6.34 / 63.8  & \textbf{6.09} / \textbf{64.0}  & \textbf{6.09} / \textbf{64.0} \\
\midrule
\multirow{2}{*}{\shortstack[l]{Qwen3-8B\\{(FP16: 8.58 / 65.8)}}}
  & 3.25 &  10.57 / 49.8  & \textbf{9.25} / 62.5  & 9.29 / \textbf{63.1} \\
  & 3.75 &  9.24 / 62.9  & 8.94 / \textbf{64.2}  & \textbf{8.89} / 63.9 \\
\midrule
\multirow{2}{*}{\shortstack[l]{Qwen3-14B\\{(FP16: 7.58 / 69.7)}}}
  & 3.25 &  8.70 / 58.8  & \textbf{8.12} / \textbf{66.8}  & 8.13 / \textbf{66.8} \\
  & 3.75 &  8.14 / 67.1  & 7.87 / 68.3  & \textbf{7.86} / \textbf{68.4} \\
\bottomrule
\end{tabular}
\end{table*}

\section{Additional Theoretical Results}
\label{app:additional_theoretical_results}

This appendix provides additional theoretical support for the activation-aware proxy and the cross-layer extension.  We use the notation from the main paper. For a module $(l,m)$ and a bit-width candidate $q$, let $\dW^{(l,m,q)}$ represent the quantization perturbation. The activation-aware self-cost used by CASA is
\begin{equation}
  \ell_{lmq} =
  \tr\left(
    \mB^{(l,m)}\dW^{(l,m,q)}
    \mA^{(l,m)}\bigl(\dW^{(l,m,q)}\bigr)^\top
  \right).
  \label{eq:app_self_cost_repeat}
\end{equation}
When a statement concerns a single module, we suppress the superscript $(l,m)$ and write $\mA$, $\mB$, and $E$.

\subsection{A Lower Bound Showing Why Scalar Proxies Can Fail}
\label{app:scalar_lower_bound}

\begin{definition}[Activation-aware and scalar quadratic proxies]
\label{def:app_quadratic_proxies}
Let $\mA\in\mathbb{R}^{d_{\rm in}\times d_{\rm in}}$ and $\mB\in\mathbb{R}^{d_{\rm out}\times d_{\rm out}}$ be symmetric positive semidefinite matrices. For a perturbation $E\in\mathbb{R}^{d_{\rm out}\times d_{\rm in}}$, define the activation-aware quadratic proxy
\begin{equation}
  Q(E):=\tr(\mB E\mA E^\top).
  \label{eq:app_Q_full}
\end{equation}
For a scalar $\alpha>0$, define the scalar proxy
\begin{equation}
  Q_\alpha(E):=\alpha\|E\|_F^2.
  \label{eq:app_Q_scalar}
\end{equation}
The scalar $\alpha$ may include any layer-wise sensitivity coefficient, normalization by $\|\mW\|_F^2$, or other scalar rescaling. Thus, \eqref{eq:app_Q_scalar} represents any proxy that assigns the same weight to all directions of $E$.
\end{definition}

\appscalarsharpdistortion*

\begin{proof}
Since $\mA$ and $\mB$ are real symmetric positive definite matrices, they admit orthogonal eigendecompositions
\begin{equation}
  \mA=\bm{U}\operatorname{diag}(a_1,\ldots,a_{d_{\rm in}})\bm{U}^\top,~~
  \mB=\bm{V}\operatorname{diag}(b_1,\ldots,b_{d_{\rm out}})\bm{V}^\top,
\end{equation}
where $0<a_{\min}:=\lambda_{\min}(\mA)\le a_j\le a_{\max}:=\lambda_{\max}(\mA)$ for every $j$, and $0<b_{\min}:=\lambda_{\min}(\mB)\le b_i\le b_{\max}:=\lambda_{\max}(\mB)$ for every $i$.

Fix any nonzero $E\in\mathbb{R}^{d_{\rm out}\times d_{\rm in}}$ and define $F:=\bm{V}^\top E\bm{U}$. Since $\bm{U}$ and $\bm{V}$ are orthogonal,
\begin{align}
  \|F\|_F^2
  &=\tr(F^\top F) \\
  &=\tr\left(\bm{U}^\top E^\top\bm{V}\bm{V}^\top E\bm{U}\right) \\
  &=\tr(E^\top E) \\
  &=\|E\|_F^2.
  \label{eq:app_fro_invariance}
\end{align}
Using cyclic invariance of the trace,
\begin{align}
  Q(E)
  &=
  \tr\left(
    \bm{V}\operatorname{diag}(b_i)\bm{V}^\top
    E
    \bm{U}\operatorname{diag}(a_j)\bm{U}^\top
    E^\top
  \right) \\
  &=
  \tr\left(
    \operatorname{diag}(b_i)
    \bm{V}^\top E\bm{U}
    \operatorname{diag}(a_j)
    \bm{U}^\top E^\top\bm{V}
  \right) \\
  &=
  \tr\left(
    \operatorname{diag}(b_i)F\operatorname{diag}(a_j)F^\top
  \right) \\
  &=
  \sum_{i=1}^{d_{\rm out}}
  \sum_{j=1}^{d_{\rm in}}
  b_i a_j F_{ij}^2.
  \label{eq:app_Q_weighted_entries}
\end{align}
Since $E\neq0$, equation \eqref{eq:app_fro_invariance} implies $F\neq0$. Define $\omega_{ij}:=F_{ij}^2/\|F\|_F^2$. Then $\omega_{ij}\ge0$ and $\sum_{i=1}^{d_{\rm out}}\sum_{j=1}^{d_{\rm in}}\omega_{ij}=1$. Dividing \eqref{eq:app_Q_weighted_entries} by $\|E\|_F^2=\|F\|_F^2$ gives
\begin{equation}
  r(E)
  :=
  \frac{Q(E)}{\|E\|_F^2} =
  \sum_{i=1}^{d_{\rm out}}
  \sum_{j=1}^{d_{\rm in}}
  \omega_{ij}b_i a_j.
  \label{eq:app_rayleigh_weighted_average}
\end{equation}
Thus $r(E)$ is a convex combination of the numbers $b_i a_j$. Therefore
\begin{equation}
  m:=a_{\min}b_{\min}
  \le
  r(E)
  \le
  a_{\max}b_{\max}=:M
  ~~
  \text{for every }E\neq0.
  \label{eq:app_rayleigh_interval}
\end{equation}
Both endpoints in \eqref{eq:app_rayleigh_interval} are attainable. Let $u_{\min}$ be a unit eigenvector of $\mA$ associated with $a_{\min}$, and let $v_{\min}$ be a unit eigenvector of $\mB$ associated with $b_{\min}$. Set $E_{\min}:=v_{\min}u_{\min}^\top$. Then $\|E_{\min}\|_F^2=1$ and
\begin{align}
  Q(E_{\min})
  &=
  \tr\left(
    \mB v_{\min}u_{\min}^\top
    \mA u_{\min}v_{\min}^\top
  \right) \\
  &=
  \tr\left(
    \mB v_{\min}
    (u_{\min}^\top\mA u_{\min})
    v_{\min}^\top
  \right) \\
  &=
  a_{\min}\tr\left(\mB v_{\min}v_{\min}^\top\right) \\
  &=
  a_{\min}\tr\left(v_{\min}^\top\mB v_{\min}\right) \\
  &=
  a_{\min}b_{\min}=m.
\end{align}
Thus $r(E_{\min})=m$. Similarly, if $u_{\max}$ and $v_{\max}$ are unit eigenvectors associated with $a_{\max}$ and $b_{\max}$, then $E_{\max}:=v_{\max}u_{\max}^\top$ satisfies $r(E_{\max})=M$.

Now fix $\alpha>0$. Since $Q_\alpha(E)=\alpha\|E\|_F^2$, for every nonzero $E$,
\begin{equation}
  \frac{Q(E)}{Q_\alpha(E)}=\frac{r(E)}{\alpha},
  ~~
  \frac{Q_\alpha(E)}{Q(E)}=\frac{\alpha}{r(E)}.
\end{equation}
The bounds $m\le r(E)\le M$ imply
\begin{equation}
  \max\left\{
    \frac{r(E)}{\alpha},
    \frac{\alpha}{r(E)}
  \right\}
  \le
  \max\left\{
    \frac{M}{\alpha},
    \frac{\alpha}{m}
  \right\}
  ~~
  \text{for all }E\neq0.
\end{equation}
Conversely, the perturbation $E_{\max}$ attains the value $M/\alpha$ in the first ratio, and $E_{\min}$ attains the value $\alpha/m$ in the second ratio. Therefore, for this fixed $\alpha$,
\begin{equation}
  D(\alpha)
  :=
  \sup_{E\neq0}
  \max\left\{
    \frac{Q(E)}{Q_\alpha(E)},
    \frac{Q_\alpha(E)}{Q(E)}
  \right\} =
  \max\left\{
    \frac{M}{\alpha},
    \frac{\alpha}{m}
  \right\}.
  \label{eq:app_D_alpha}
\end{equation}
For any $\alpha>0$,
\begin{equation}
  D(\alpha)
  \ge
  \sqrt{\frac{M}{\alpha}\cdot\frac{\alpha}{m}} =
  \sqrt{\frac{M}{m}}.
\end{equation}
Equality is obtained when $M/\alpha=\alpha/m$, namely when $\alpha=\sqrt{Mm}$. Hence
\begin{align}
  \inf_{\alpha>0}D(\alpha)
  &=
  \sqrt{\frac{M}{m}} \\
  &=
  \sqrt{\frac{a_{\max}b_{\max}}{a_{\min}b_{\min}}} \\
  &=
  \sqrt{\frac{a_{\max}}{a_{\min}}\frac{b_{\max}}{b_{\min}}} \\
  &=
  \sqrt{\kappa(\mA)\kappa(\mB)}.
\end{align}
This proves the theorem.
\end{proof}

\begin{corollary}[Singular curvature gives infinite scalar distortion]
\label{cor:app_singular_infinite_distortion}
Let $\mA\succeq0$ and $\mB\succeq0$. Use the convention that $c/0=+\infty$ for every $c>0$. If either $\mA$ or $\mB$ is singular, then for all $\alpha>0$,
\begin{equation}
  \sup_{E\neq0}
  \max\left\{
    \frac{Q(E)}{Q_\alpha(E)},
    \frac{Q_\alpha(E)}{Q(E)}
  \right\} =+\infty.
\end{equation}
\end{corollary}

\begin{proof}
Assume first that $\mA$ is singular. Then there exists a unit vector $u_0\in\mathbb{R}^{d_{\rm in}}$ such that $\mA u_0=0$. Let $v\in\mathbb{R}^{d_{\rm out}}$ be any unit vector and set $E_0:=vu_0^\top$. Then $E_0\neq0$, $\|E_0\|_F^2=1$, and
\begin{align}
  E_0\mA E_0^\top
  &=
  vu_0^\top\mA u_0v^\top \\
  &=
  v(u_0^\top\mA u_0)v^\top \\
  &=0.
\end{align}
Consequently,
\begin{equation}
  Q(E_0)=\tr(\mB E_0\mA E_0^\top),
  ~~
  Q_\alpha(E_0)=\alpha\|E_0\|_F^2=\alpha>0.
\end{equation}
Thus $Q(E_0)=0$, and therefore $Q_\alpha(E_0)/Q(E_0)=+\infty$.

If instead $\mB$ is singular, choose a unit vector $v_0$ such that $\mB v_0=0$, choose any unit vector $u$, and set $E_0:=v_0u^\top$. Then
\begin{align}
  Q(E_0)
  &=
  \tr\left(\mB v_0u^\top\mA u v_0^\top\right) \\
  &=
  (u^\top\mA u)\tr\left(\mB v_0v_0^\top\right) \\
  &=
  (u^\top\mA u)\tr\left(v_0^\top\mB v_0\right) \\
  &=0,
\end{align}
while $Q_\alpha(E_0)=\alpha>0$. Hence, the supremum is again infinite.
\end{proof}

We presents a rigorous statement and proof of Proposition~\ref{prop:app_misranking_mckp_regret}. 
We first formally restate Proposition~\ref{prop:app_misranking_mckp_regret} below.

\begin{proposition}[Proxy misranking can create arbitrarily large MCKP regret]
Consider two modules indexed by $i\in\{1,2\}$ and two bit-width choices, $L$ and $H$, where $L$ is of lower precision and $H$ is of higher precision. Suppose the bit budget allows exactly one module to use $H$. Let the true second-order cost be
\begin{equation}
  C_i(b)=\gamma_i d_b,
  ~~ b\in\{L,H\},
\end{equation}
where $d_L>d_H>0$ and $\gamma_i>0$. Let a proxy MCKP use the proxy cost, $\widehat C_i(b)=\widehat\gamma_i d_b$, where $\widehat\gamma_i>0$. Suppose that the true sensitivities and proxy sensitivities have opposite rankings: $\gamma_1>\gamma_2,~~ \widehat\gamma_2>\widehat\gamma_1$.
Then the true optimal allocation assigns $H$ to module $1$, whereas the proxy-optimal allocation assigns $H$ to module $2$. The true regret of the proxy-optimal allocation is
\begin{equation}
  \operatorname{Regret} =
  (\gamma_1-\gamma_2)(d_L-d_H).
\end{equation}
In particular, taking $\gamma_1=R$ and $\gamma_2=1$ makes the regret equal to $(R-1)(d_L-d_H)$, which diverges as $R\to\infty$.
\end{proposition}

\begin{proof}
Because exactly one module can use $H$, there are exactly two feasible allocations:
\begin{equation}
  A_1=(H,L),~~ A_2=(L,H),
\end{equation}
where $A_1$ assigns $H$ to module $1$ and $L$ to module $2$, while $A_2$ assigns $L$ to module $1$ and $H$ to module $2$.

First, consider the true objective. The true cost of $A_1$ is
\begin{equation}
  C(A_1)=\gamma_1d_H+\gamma_2d_L,
\end{equation}
and the true cost of $A_2$ is
\begin{equation}
  C(A_2)=\gamma_1d_L+\gamma_2d_H.
\end{equation}
Therefore,
\begin{align}
  C(A_2)-C(A_1)
  &=
  \gamma_1d_L+\gamma_2d_H-\gamma_1d_H-\gamma_2d_L \\
  &=
  \gamma_1(d_L-d_H)-\gamma_2(d_L-d_H) \\
  &=
  (\gamma_1-\gamma_2)(d_L-d_H).
\end{align}
Since $\gamma_1>\gamma_2$ and $d_L>d_H$, we have
\begin{equation}
  C(A_2)-C(A_1)>0.
\end{equation}
Thus, $C(A_1)<C(A_2)$. Since $A_1$ and $A_2$ are the only feasible allocations, $A_1$ is the unique true optimum.

Next, consider the proxy objective. The proxy cost of $A_1$ is
\begin{equation}
  \widehat C(A_1)=\widehat\gamma_1d_H+\widehat\gamma_2d_L,
\end{equation}
and the proxy cost of $A_2$ is
\begin{equation}
  \widehat C(A_2)=\widehat\gamma_1d_L+\widehat\gamma_2d_H.
\end{equation}
Subtracting gives
\begin{align}
  \widehat C(A_1)-\widehat C(A_2)
  &=
  \widehat\gamma_1d_H+\widehat\gamma_2d_L-\widehat\gamma_1d_L-\widehat\gamma_2d_H \\
  &=
  -\widehat\gamma_1(d_L-d_H)+\widehat\gamma_2(d_L-d_H) \\
  &=
  (\widehat\gamma_2-\widehat\gamma_1)(d_L-d_H).
\end{align}
Since $\widehat\gamma_2>\widehat\gamma_1$ and $d_L>d_H$, we have
\begin{equation}
  \widehat C(A_1)-\widehat C(A_2)>0.
\end{equation}
Therefore,
\begin{equation}
  \widehat C(A_2)<\widehat C(A_1).
\end{equation}
Since $A_1$ and $A_2$ are the only feasible allocations, $A_2$ is the unique proxy optimum.

The regret of the proxy-optimal allocation is defined as
\begin{equation}
  \operatorname{Regret} =
  C(A_2)-C(A_1).
\end{equation}
Using the true-cost calculation above,
\begin{equation}
  \operatorname{Regret} =
  (\gamma_1-\gamma_2)(d_L-d_H).
\end{equation}
Finally, if $\gamma_1=R$ and $\gamma_2=1$, then
\begin{equation}
  \operatorname{Regret} =
  (R-1)(d_L-d_H).
\end{equation}
Since $d_L-d_H>0$ is fixed, this quantity tends to $+\infty$ as $R\to\infty$.
\end{proof}

\subsection{High-rate Quantization and the Optimal Bit-allocation Rule}
\label{app:high_rate_bit_allocation}

Let $i=(l,m)$ denote a module index, and let $\mathcal{I}$ be the finite set of modules over which mixed precision is performed. For each module $i\in\mathcal{I}$, write
\begin{equation}
  n_i:=\text{size}_{lm}>0
\end{equation}
for the number of scalar weights in module $i$. We use $\beta_i\in\mathbb{R}$ to denote the continuous bit width assigned to module $i$, in order to avoid conflict with the diagonal entries $b_{i,r}$ of $\mB^{(i)}$. Throughout this subsection, $\operatorname{vec}$ stacks the columns of its matrix argument.

\begin{assumption}[High-rate second-moment model]
\label{assump:app_high_rate_model}
For each module $i\in\mathcal{I}$ and continuous bit width $\beta_i$, let
\begin{equation}
  E_i(\beta_i):=\dW^{(i,\beta_i)}\in\mathbb{R}^{d_{{\rm out},i}\times d_{{\rm in},i}}
\end{equation}
be the random quantization perturbation. We assume that there exists a fixed positive semidefinite matrix
\begin{equation}
  \Sigma_i\in\mathbb{R}^{d_{{\rm out},i}d_{{\rm in},i}\times d_{{\rm out},i}d_{{\rm in},i}},
  ~~
  \Sigma_i\succeq0,
\end{equation}
such that
\begin{equation}
  \mathbb{E}\left[
    \operatorname{vec}(E_i(\beta_i))
    \operatorname{vec}(E_i(\beta_i))^\top
  \right] =
  2^{-2\beta_i}\Sigma_i.
  \label{eq:app_high_rate_second_moment_complete}
\end{equation}
Thus the perturbation second moment decays proportionally to $2^{-2\beta_i}$.
\end{assumption}

The randomness in $E_i(\beta_i)$ may be interpreted either as randomness induced by stochastic rounding or as the standard high-rate quantization model in which quantization errors are analyzed under a source distribution. The results below use only the second-moment condition in Assumption~\ref{assump:app_high_rate_model}.

\begin{proposition}[Expected activation-aware distortion under the high-rate model]
\label{prop:app_high_rate_expected_distortion_complete}
Suppose Assumption~\ref{assump:app_high_rate_model} holds. Assume also that
\begin{equation}
  \mA^{(i)}\in\mathbb{R}^{d_{{\rm in},i}\times d_{{\rm in},i}},
  ~~
  \mB^{(i)}\in\mathbb{R}^{d_{{\rm out},i}\times d_{{\rm out},i}}
\end{equation}
are symmetric positive semidefinite matrices. Define
\begin{equation}
  \Gamma_i
  :=
  \tr\left(
    (\mA^{(i)}\otimes\mB^{(i)})\Sigma_i
  \right).
  \label{eq:app_Gamma_full_complete}
\end{equation}
Then
\begin{equation}
  \mathbb{E}\left[
    \tr\left(
      \mB^{(i)}E_i(\beta_i)\mA^{(i)}E_i(\beta_i)^\top
    \right)
  \right] =
  \Gamma_i2^{-2\beta_i}.
  \label{eq:app_high_rate_expected_cost_complete}
\end{equation}
Moreover, $\Gamma_i\ge0$. If, in addition, $\mA^{(i)}$ and $\mB^{(i)}$ are diagonal, with diagonal entries
\begin{equation}
  \mA^{(i)}=\operatorname{diag}(a_{i,1},\ldots,a_{i,d_{{\rm in},i}}),
  ~~
  \mB^{(i)}=\operatorname{diag}(b_{i,1},\ldots,b_{i,d_{{\rm out},i}}),
\end{equation}
and if the entrywise second moments satisfy
\begin{equation}
  \mathbb{E}\left[
    \left(\Delta w_{rc}^{(i)}(\beta_i)\right)^2
  \right] =
  \sigma_{i,rc}^2 2^{-2\beta_i},
  ~~
  1\le r\le d_{{\rm out},i},~
  1\le c\le d_{{\rm in},i},
  \label{eq:app_entrywise_second_moment_complete}
\end{equation}
then
\begin{equation}
  \Gamma_i =
  \sum_{r=1}^{d_{{\rm out},i}}
  \sum_{c=1}^{d_{{\rm in},i}}
  b_{i,r}a_{i,c}\sigma_{i,rc}^2.
  \label{eq:app_Gamma_diag_complete}
\end{equation}
\end{proposition}

\begin{proof}
Fix a module $i\in\mathcal{I}$ and a bit width $\beta_i$. For notational brevity, write
\begin{equation}
  E:=E_i(\beta_i),
  ~~
  \mA:=\mA^{(i)},
  ~~
  \mB:=\mB^{(i)},
  ~~
  z:=\operatorname{vec}(E).
\end{equation}
Since $E\in\mathbb{R}^{d_{{\rm out},i}\times d_{{\rm in},i}}$, we have $z\in\mathbb{R}^{d_{{\rm out},i}d_{{\rm in},i}}$. The standard vectorization identity states that, for compatible matrices $E$, $\mA$, and $\mB$,
\begin{equation}
  \tr(\mB E\mA E^\top) =
  \operatorname{vec}(E)^\top(\mA^\top\otimes\mB)\operatorname{vec}(E).
  \label{eq:app_vec_identity_complete}
\end{equation}
Because $\mA$ is symmetric, $\mA^\top=\mA$. Hence, if we define
\begin{equation}
  M:=\mA\otimes\mB,
\end{equation}
then \eqref{eq:app_vec_identity_complete} gives
\begin{equation}
  \tr(\mB E\mA E^\top) =
  z^\top Mz.
  \label{eq:app_quadratic_form_complete}
\end{equation}
We next recall the elementary identity
\begin{equation}
  \mathbb{E}[z^\top Mz] =
  \tr\left(M\mathbb{E}[zz^\top]\right).
  \label{eq:app_quadratic_expectation_identity_complete}
\end{equation}
Indeed, since $z^\top Mz$ is a scalar, it is equal to its trace. Therefore,
\begin{align}
  \mathbb{E}[z^\top Mz]
  &=
  \mathbb{E}\left[\tr(z^\top Mz)\right] \\
  &=
  \mathbb{E}\left[\tr(Mzz^\top)\right] \\
  &=
  \tr\left(M\mathbb{E}[zz^\top]\right),
\end{align}
where the second equality uses the cyclic invariance of the trace and the third equality uses linearity of expectation and trace. Applying \eqref{eq:app_quadratic_expectation_identity_complete} to \eqref{eq:app_quadratic_form_complete} yields
\begin{equation}
  \mathbb{E}\left[
    \tr(\mB E\mA E^\top)
  \right] =
  \tr\left(M\mathbb{E}[zz^\top]\right).
  \label{eq:app_expected_cost_pre_substitution_complete}
\end{equation}
By Assumption~\ref{assump:app_high_rate_model},
\begin{equation}
  \mathbb{E}[zz^\top] =
  2^{-2\beta_i}\Sigma_i.
\end{equation}
Substituting this into \eqref{eq:app_expected_cost_pre_substitution_complete}, and using $M=\mA^{(i)}\otimes\mB^{(i)}$, gives
\begin{align}
  \mathbb{E}\left[
    \tr\left(
      \mB^{(i)}E_i(\beta_i)\mA^{(i)}E_i(\beta_i)^\top
    \right)
  \right]
  &=
  \tr\left(
    (\mA^{(i)}\otimes\mB^{(i)})2^{-2\beta_i}\Sigma_i
  \right) \\
  &=
  2^{-2\beta_i}
  \tr\left(
    (\mA^{(i)}\otimes\mB^{(i)})\Sigma_i
  \right) \\
  &=
  \Gamma_i2^{-2\beta_i}.
\end{align}
This proves \eqref{eq:app_high_rate_expected_cost_complete}.

It remains to prove nonnegativity of $\Gamma_i$. Since $\mA^{(i)}\succeq0$ and $\mB^{(i)}\succeq0$, the Kronecker product satisfies
\begin{equation}
  \mA^{(i)}\otimes\mB^{(i)}\succeq0.
\end{equation}
Set
\begin{equation}
  P:=\mA^{(i)}\otimes\mB^{(i)}.
\end{equation}
Then $P\succeq0$, so its unique positive semidefinite square root $P^{1/2}$ exists. Since $\Sigma_i\succeq0$, we have
\begin{equation}
  P^{1/2}\Sigma_iP^{1/2}\succeq0.
\end{equation}
Therefore its trace is nonnegative. Using cyclic invariance of the trace,
\begin{equation}
  \Gamma_i =
  \tr(P\Sigma_i) =
  \tr(P^{1/2}P^{1/2}\Sigma_i) =
  \tr(P^{1/2}\Sigma_iP^{1/2})
  \ge0.
\end{equation}
This proves $\Gamma_i\ge0$.

Finally, suppose $\mA^{(i)}$ and $\mB^{(i)}$ are diagonal as stated. For this diagonal case, again abbreviate $E=E_i(\beta_i)$ and write its entries as $E_{rc}=\Delta w_{rc}^{(i)}(\beta_i)$. The $(r,r)$ entry of $E\mA^{(i)}E^\top$ is
\begin{equation}
  (E\mA^{(i)}E^\top)_{rr} =
  \sum_{c=1}^{d_{{\rm in},i}}a_{i,c}E_{rc}^2.
\end{equation}
Since $\mB^{(i)}$ is diagonal, we obtain
\begin{align}
  \tr\left(\mB^{(i)}E\mA^{(i)}E^\top\right)
  &=
  \sum_{r=1}^{d_{{\rm out},i}}b_{i,r}(E\mA^{(i)}E^\top)_{rr} \\
  &=
  \sum_{r=1}^{d_{{\rm out},i}}
  \sum_{c=1}^{d_{{\rm in},i}}
  b_{i,r}a_{i,c}E_{rc}^2 \\
  &=
  \sum_{r=1}^{d_{{\rm out},i}}
  \sum_{c=1}^{d_{{\rm in},i}}
  b_{i,r}a_{i,c}
  \left(\Delta w_{rc}^{(i)}(\beta_i)\right)^2.
\end{align}
Taking expectations and using \eqref{eq:app_entrywise_second_moment_complete} gives
\begin{align}
  \mathbb{E}\left[
    \tr\left(\mB^{(i)}E_i(\beta_i)\mA^{(i)}E_i(\beta_i)^\top\right)
  \right]
  &=
  \sum_{r=1}^{d_{{\rm out},i}}
  \sum_{c=1}^{d_{{\rm in},i}}
  b_{i,r}a_{i,c}
  \mathbb{E}\left[
    \left(\Delta w_{rc}^{(i)}(\beta_i)\right)^2
  \right] \\
  &=
  \sum_{r=1}^{d_{{\rm out},i}}
  \sum_{c=1}^{d_{{\rm in},i}}
  b_{i,r}a_{i,c}\sigma_{i,rc}^2 2^{-2\beta_i} \\
  &=
  \left(
    \sum_{r=1}^{d_{{\rm out},i}}
    \sum_{c=1}^{d_{{\rm in},i}}
    b_{i,r}a_{i,c}\sigma_{i,rc}^2
  \right)2^{-2\beta_i}.
\end{align}
Comparing this expression with \eqref{eq:app_high_rate_expected_cost_complete}, we conclude that
\begin{equation}
  \Gamma_i =
  \sum_{r=1}^{d_{{\rm out},i}}
  \sum_{c=1}^{d_{{\rm in},i}}
  b_{i,r}a_{i,c}\sigma_{i,rc}^2.
\end{equation}
This proves \eqref{eq:app_Gamma_diag_complete}.
\end{proof}

In the diagonal case, the quantity $\tr(\mB^{(i)}E_i(\beta_i)\mA^{(i)}E_i(\beta_i)^\top)$ depends only on the squared entries $(\Delta w_{rc}^{(i)}(\beta_i))^2$. Hence the diagonal formula requires only the entrywise second moments in \eqref{eq:app_entrywise_second_moment_complete}; it does not require the off-diagonal covariances $\mathbb{E}[\Delta w_{rc}^{(i)}(\beta_i)\Delta w_{r'c'}^{(i)}(\beta_i)]$ to vanish.

\begin{assumption}[Positive effective sensitivity]
\label{assump:app_positive_gamma}
For every module $i\in\mathcal{I}$ considered in the continuous bit-allocation problem, the effective activation-aware high-rate sensitivity satisfies
\begin{equation}
  \Gamma_i>0.
\end{equation}
\end{assumption}

\appoptimalcontinuousbits*

\begin{proof}
We first prove strict convexity. For each $i\in\mathcal{I}$, define
\begin{equation}
  f_i(\beta_i):=\Gamma_i2^{-2\beta_i}.
\end{equation}
Then
\begin{equation}
  \frac{d}{d\beta_i}f_i(\beta_i) =
  -2\ln(2)\Gamma_i2^{-2\beta_i},
\end{equation}
and
\begin{equation}
  \frac{d^2}{d\beta_i^2}f_i(\beta_i) =
  4(\ln 2)^2\Gamma_i2^{-2\beta_i}.
\end{equation}
Since $\Gamma_i>0$ and $2^{-2\beta_i}>0$, we have
\begin{equation}
  \frac{d^2}{d\beta_i^2}f_i(\beta_i)>0
\end{equation}
for every $\beta_i\in\mathbb{R}$. Hence each $f_i$ is strictly convex. Since $F(\beta)=\sum_i f_i(\beta_i)$ is a sum of strictly convex functions in separate coordinates, $F$ is strictly convex on $\mathbb{R}^{|\mathcal{I}|}$. The feasible set
\begin{equation}
  \left\{\beta\in\mathbb{R}^{|\mathcal{I}|}:\sum_{i\in\mathcal{I}}n_i\beta_i\le B_{\rm total}\right\}
\end{equation}
is convex because it is a half-space.

We next prove existence of a minimizer. The feasible set is nonempty; for example, taking all $\beta_i$ sufficiently negative gives $\sum_i n_i\beta_i\le B_{\rm total}$. We first show that any optimal point, if it exists, must satisfy the budget equality. Suppose that $\beta$ is feasible and satisfies
\begin{equation}
  \sum_i n_i\beta_i<B_{\rm total}.
\end{equation}
Choose any $k\in\mathcal{I}$. Since the inequality is strict and $n_k>0$, there exists $\varepsilon>0$ sufficiently small such that the vector $\widetilde{\beta}$ defined by
\begin{equation}
  \widetilde{\beta}_k=\beta_k+\varepsilon,
  ~~
  \widetilde{\beta}_i=\beta_i~\text{for }i\neq k
\end{equation}
is still feasible. Moreover, since $f_k$ is strictly decreasing, we have
\begin{equation}
  f_k(\widetilde{\beta}_k)<f_k(\beta_k),
\end{equation}
while all other terms in $F$ are unchanged. Therefore
\begin{equation}
  F(\widetilde{\beta})<F(\beta).
\end{equation}
Thus no point with slack budget can be optimal. Consequently, any minimizer must lie on the affine hyperplane
\begin{equation}
  \mathcal{H}:=
  \left\{
    \beta\in\mathbb{R}^{|\mathcal{I}|}:
    \sum_i n_i\beta_i=B_{\rm total}
  \right\}.
\end{equation}

It remains to show that a minimizer exists on $\mathcal{H}$. Choose any $\bar{\beta}\in\mathcal{H}$, for instance
\begin{equation}
  \bar{\beta}_i=\frac{B_{\rm total}}{\sum_j n_j}
  ~~
  \text{for all } i\in\mathcal{I}.
\end{equation}
Consider the sublevel set
\begin{equation}
  \mathcal{S}:=
  \left\{
    \beta\in\mathcal{H}:F(\beta)\le F(\bar{\beta})
  \right\}.
\end{equation}
The set $\mathcal{S}$ is closed because $\mathcal{H}$ is closed and $F$ is continuous. We now show that $\mathcal{S}$ is bounded. Suppose, for contradiction, that $\mathcal{S}$ is unbounded. Then there exists a sequence $\{\beta^{(t)}\}_{t=1}^{\infty}\subset\mathcal{S}$ such that
\begin{equation}
  \|\beta^{(t)}\|_2\to\infty.
\end{equation}
Since $\mathcal{I}$ is finite, after passing to a subsequence, at least one coordinate is unbounded in absolute value. If some coordinate satisfies $\beta_k^{(t)}\to-\infty$ along a subsequence, then
\begin{equation}
  \Gamma_k2^{-2\beta_k^{(t)}}\to+\infty,
\end{equation}
and hence
\begin{equation}
  F(\beta^{(t)})\to+\infty,
\end{equation}
contradicting $F(\beta^{(t)})\le F(\bar{\beta})$. Otherwise, some coordinate satisfies $\beta_k^{(t)}\to+\infty$ along a subsequence. Since every $\beta^{(t)}\in\mathcal{H}$,
\begin{equation}
  \sum_i n_i\beta_i^{(t)}=B_{\rm total}.
\end{equation}
Because $n_k>0$ and $n_i>0$ for all $i$, the positive divergence of $n_k\beta_k^{(t)}$ must be offset by at least one coordinate $r\neq k$ satisfying $\beta_r^{(t)}\to-\infty$ along a further subsequence. Then again
\begin{equation}
  \Gamma_r2^{-2\beta_r^{(t)}}\to+\infty,
\end{equation}
so
\begin{equation}
  F(\beta^{(t)})\to+\infty,
\end{equation}
contradicting $F(\beta^{(t)})\le F(\bar{\beta})$. Therefore $\mathcal{S}$ is bounded. Since $\mathcal{S}$ is closed and bounded in finite-dimensional Euclidean space, it is compact. By continuity of $F$, there exists a minimizer of $F$ over $\mathcal{S}$, and therefore over $\mathcal{H}$. Since every global minimizer must lie on $\mathcal{H}$, this minimizer is also a global minimizer of the original inequality-constrained problem. Strict convexity of $F$ and convexity of the feasible set imply that the global minimizer is unique.

We now derive its closed form. The problem is convex and satisfies Slater's condition because, for example, taking all $\beta_i$ sufficiently negative yields
\begin{equation}
  \sum_i n_i\beta_i<B_{\rm total}.
\end{equation}
Hence the Karush--Kuhn--Tucker conditions are necessary and sufficient for global optimality. Let $\lambda\ge0$ be the Lagrange multiplier for the constraint
\begin{equation}
  \sum_i n_i\beta_i\le B_{\rm total}.
\end{equation}
The Lagrangian is
\begin{equation}
  \mathcal{L}(\beta,\lambda) =
  \sum_i\Gamma_i2^{-2\beta_i}
  +
  \lambda\left(\sum_i n_i\beta_i-B_{\rm total}\right).
\end{equation}
Stationarity at the unique minimizer $\beta^{\prime}$ gives, for every $i\in\mathcal{I}$,
\begin{equation}
  -2\ln(2)\Gamma_i2^{-2\beta_i^{\prime}}+\lambda n_i=0.
  \label{eq:app_stationarity_beta}
\end{equation}
Since the budget is active, complementary slackness is consistent with any $\lambda\ge0$. However, \eqref{eq:app_stationarity_beta} implies $\lambda>0$, because $\Gamma_i>0$, $n_i>0$, and $2^{-2\beta_i^{\prime}}>0$. Rearranging \eqref{eq:app_stationarity_beta} gives
\begin{equation}
  2^{-2\beta_i^{\prime}} =
  \frac{\lambda n_i}{2\ln(2)\Gamma_i}.
  \label{eq:app_two_power_solution}
\end{equation}
Taking base-two logarithms of both sides yields
\begin{equation}
  -2\beta_i^{\prime} =
  \log_2\lambda+\log_2 n_i-\log_2(2\ln(2))-\log_2\Gamma_i.
\end{equation}
Multiplying by $-1/2$ gives
\begin{equation}
  \beta_i^{\prime} =
  \frac{1}{2}\log_2\left(\frac{\Gamma_i}{n_i}\right)
  +
  \frac{1}{2}\log_2\left(\frac{2\ln(2)}{\lambda}\right).
\end{equation}
The second term is independent of $i$. Define
\begin{equation}
  c:=
  \frac{1}{2}\log_2\left(\frac{2\ln(2)}{\lambda}\right).
\end{equation}
Then
\begin{equation}
  \beta_i^{\prime} =
  \frac{1}{2}\log_2\left(\frac{\Gamma_i}{n_i}\right)+c.
\end{equation}
Substituting this expression into the active budget equality gives
\begin{align}
  B_{\rm total}
  &=
  \sum_i n_i\beta_i^{\prime} \\
  &=
  \sum_i n_i
  \left[
    \frac{1}{2}\log_2\left(\frac{\Gamma_i}{n_i}\right)+c
  \right] \\
  &=
  \frac{1}{2}\sum_i n_i\log_2\left(\frac{\Gamma_i}{n_i}\right)
  +
  c\sum_i n_i.
\end{align}
Solving this equation for $c$ gives
\begin{equation}
  c =
  \frac{
    B_{\rm total}
    -
    \frac{1}{2}\sum_{j\in\mathcal{I}}n_j
    \log_2\left(\frac{\Gamma_j}{n_j}\right)
  }{
    \sum_{j\in\mathcal{I}}n_j
  }.
\end{equation}
Finally, subtracting the formula for $\beta_j^{\prime}$ from the formula for $\beta_i^{\prime}$ cancels the common constant $c$ and gives
\begin{equation}
  \beta_i^{\prime}-\beta_j^{\prime} =
  \frac{1}{2}\log_2\left(
    \frac{\Gamma_i/n_i}{\Gamma_j/n_j}
  \right).
\end{equation}
This completes the proof.
\end{proof}

\begin{corollary}[Bounded bit widths give clipped water-filling]
\label{cor:app_clipped_water_filling}
Assume $\Gamma_i>0$ and $n_i>0$ for all $i\in\mathcal{I}$. Suppose practical constraints impose
\begin{equation}
  \beta_i^{\min}\le\beta_i\le\beta_i^{\max}
\end{equation}
for every $i\in\mathcal{I}$, with $\beta_i^{\min}\le\beta_i^{\max}$, and suppose the budget satisfies
\begin{equation}
  \sum_i n_i\beta_i^{\min}
  <
  B_{\rm total}
  <
  \sum_i n_i\beta_i^{\max}.
  \label{eq:app_budget_strict_box_feasible}
\end{equation}
Then the problem
\begin{equation}
  \begin{aligned}
    \min_{\beta}~
    &\sum_i\Gamma_i2^{-2\beta_i} \\
    \text{subject to}~
    &\sum_i n_i\beta_i\le B_{\rm total}, \\
    &\beta_i^{\min}\le\beta_i\le\beta_i^{\max}
    ~\forall i\in\mathcal{I}
  \end{aligned}
  \label{eq:app_bounded_problem}
\end{equation}
has a unique global minimizer. Moreover, there exists a constant $c\in\mathbb{R}$ such that the unique minimizer is
\begin{equation}
  \beta_i^{\prime} =
  \min\left\{
    \beta_i^{\max},
    \max\left\{
      \beta_i^{\min},
      \frac{1}{2}\log_2\left(\frac{\Gamma_i}{n_i}\right)+c
    \right\}
  \right\},
  \label{eq:app_clipped_water_filling}
\end{equation}
where $c$ is chosen so that
\begin{equation}
  \sum_i n_i\beta_i^{\prime}=B_{\rm total}.
  \label{eq:app_clipped_budget_active}
\end{equation}
\end{corollary}

\begin{proof}
The objective is strictly convex by the same argument as in Theorem~\ref{thm:app_optimal_continuous_bits}. The feasible set is the intersection of a closed box and a closed half-space, hence it is closed and bounded. It is nonempty because \eqref{eq:app_budget_strict_box_feasible} implies that the lower-bound vector $\beta^{\min}$ is feasible. Therefore, by continuity, a global minimizer exists. Since the objective is strictly convex and the feasible set is convex, the global minimizer is unique.

We next show that the budget constraint is active at the minimizer. Let $\beta^{\prime}$ be the unique minimizer. Suppose, for contradiction, that
\begin{equation}
  \sum_i n_i\beta_i^{\prime}<B_{\rm total}.
\end{equation}
If $\beta_i^{\prime}=\beta_i^{\max}$ for every $i$, then
\begin{equation}
  \sum_i n_i\beta_i^{\prime} =
  \sum_i n_i\beta_i^{\max}
  >
  B_{\rm total},
\end{equation}
this contradicts feasibility. Hence there exists at least one index $k$ such that
\begin{equation}
  \beta_k^{\prime}<\beta_k^{\max}.
\end{equation}
Since the budget inequality is strict and $\beta_k^{\prime}<\beta_k^{\max}$, there exists $\varepsilon>0$ sufficiently small such that the vector $\widetilde{\beta}$ defined by
\begin{equation}
  \widetilde{\beta}_k=\beta_k^{\prime}+\varepsilon,
  ~~
  \widetilde{\beta}_i=\beta_i^{\prime}
  ~\text{for }i\neq k
\end{equation}
satisfies both the box constraints and the budget constraint. Since $\Gamma_k2^{-2\beta_k}$ is strictly decreasing in $\beta_k$, we have
\begin{equation}
  F(\widetilde{\beta})<F(\beta^{\prime}),
\end{equation}
contradicting optimality. Therefore
\begin{equation}
  \sum_i n_i\beta_i^{\prime}=B_{\rm total}.
\end{equation}

We now derive the clipped form. For any $\lambda>0$, define the box-constrained Lagrangian subproblem
\begin{equation}
  \min_{\beta_i^{\min}\le\beta_i\le\beta_i^{\max}}
  \sum_i
  \left[
    \Gamma_i2^{-2\beta_i}
    +
    \lambda n_i\beta_i
  \right].
\end{equation}
This problem separates across coordinates. For each $i$, define
\begin{equation}
  \phi_i(\beta_i):=
  \Gamma_i2^{-2\beta_i}+\lambda n_i\beta_i.
\end{equation}
Its derivative is
\begin{equation}
  \phi_i'(\beta_i) =
  -2\ln(2)\Gamma_i2^{-2\beta_i}+\lambda n_i,
\end{equation}
and its second derivative is
\begin{equation}
  \phi_i''(\beta_i) =
  4(\ln 2)^2\Gamma_i2^{-2\beta_i}>0.
\end{equation}
Thus $\phi_i$ is strictly convex. The unconstrained minimizer is the unique point satisfying
\begin{equation}
  -2\ln(2)\Gamma_i2^{-2\beta_i}+\lambda n_i=0.
\end{equation}
Equivalently,
\begin{equation}
  2^{-2\beta_i} =
  \frac{\lambda n_i}{2\ln(2)\Gamma_i}.
\end{equation}
Taking base-two logarithms gives
\begin{equation}
  \beta_i^{\rm unc} =
  \frac{1}{2}\log_2\left(\frac{\Gamma_i}{n_i}\right)
  +
  \frac{1}{2}\log_2\left(\frac{2\ln(2)}{\lambda}\right).
\end{equation}
Since $\phi_i$ is strictly convex, its minimizer over the interval $[\beta_i^{\min},\beta_i^{\max}]$ is the projection of $\beta_i^{\rm unc}$ onto that interval. Therefore, if we define
\begin{equation}
  c:=
  \frac{1}{2}\log_2\left(\frac{2\ln(2)}{\lambda}\right),
\end{equation}
then the coordinate-wise minimizer is
\begin{equation}
  \beta_i(c) =
  \min\left\{
    \beta_i^{\max},
    \max\left\{
      \beta_i^{\min},
      \frac{1}{2}\log_2\left(\frac{\Gamma_i}{n_i}\right)+c
    \right\}
  \right\}.
\end{equation}

It remains to show that one can choose $c$ so that the active budget is satisfied. Define
\begin{equation}
  G(c)
  :=
  \sum_i n_i
  \min\left\{
    \beta_i^{\max},
    \max\left\{
      \beta_i^{\min},
      \frac{1}{2}\log_2\left(\frac{\Gamma_i}{n_i}\right)+c
    \right\}
  \right\}.
\end{equation}
Each summand is continuous and nondecreasing in $c$, so $G$ is continuous and nondecreasing. Moreover,
\begin{equation}
  \lim_{c\to-\infty}G(c) =
  \sum_i n_i\beta_i^{\min},
\end{equation}
and
\begin{equation}
  \lim_{c\to+\infty}G(c) =
  \sum_i n_i\beta_i^{\max}.
\end{equation}
By \eqref{eq:app_budget_strict_box_feasible} and the intermediate value theorem, there exists $c\in\mathbb{R}$ such that
\begin{equation}
  G(c)=B_{\rm total}.
\end{equation}
Let $\beta(c)$ denote the vector defined by the clipped formula with this value of $c$, and let
\begin{equation}
  \lambda:=2\ln(2)2^{-2c}>0.
\end{equation}
By construction, $\beta(c)$ minimizes
\begin{equation}
  \sum_i
  \left[
    \Gamma_i2^{-2\beta_i}
    +
    \lambda n_i\beta_i
  \right]
\end{equation}
over the box constraints. Therefore, for every box-feasible $\beta$,
\begin{equation}
  F(\beta)+\lambda\sum_i n_i\beta_i
  \ge
  F(\beta(c))+\lambda\sum_i n_i\beta_i(c).
\end{equation}
Now let $\beta$ be feasible for the original problem. Then
\begin{equation}
  \sum_i n_i\beta_i\le B_{\rm total} =
  \sum_i n_i\beta_i(c).
\end{equation}
Using $\lambda>0$, we obtain
\begin{align}
  F(\beta)
  &\ge
  F(\beta(c))
  +
  \lambda\sum_i n_i\beta_i(c)
  -
  \lambda\sum_i n_i\beta_i \\
  &=
  F(\beta(c))
  +
  \lambda
  \left(
    B_{\rm total}
    -
    \sum_i n_i\beta_i
  \right) \\
  &\ge
  F(\beta(c)).
\end{align}
Thus $\beta(c)$ is globally optimal for the original bounded problem. Since the global minimizer is unique, $\beta(c)$ is the unique minimizer. This proves the clipped water-filling form.
\end{proof}

\subsection{Why cross-layer extensions are theoretically beneficial}
\label{app:cross_layer_benefit}

The main MCKP objective uses only self-costs $\ell_{lmq}$, corresponding to a block-diagonal Hessian approximation. We now formalize what is gained by retaining cross-layer Hessian blocks.

\begin{definition}[Self and cross-layer quadratic objectives]
\label{def:app_cross_objectives}
Let $i=(l,m)$ denote a module index, and let
\begin{equation}
  e_{iq}:=\operatorname{vec}(\dW^{(i,q)})
\end{equation}
be the vectorized perturbation of module $i$ under bit candidate $q$. Let $\mH_{ii}$ be the Hessian block for module $i$, and let $\mH_{ij}$ be the cross-Hessian block between modules $i$ and $j$. The full Hessian is symmetric, so $\mH_{ji}=\mH_{ij}^\top$. Let $\mathcal{E}$ be a set of unordered interacting pairs, for example adjacent pairs $((l,m),(l+1,m))$, each included once.

For a feasible assignment $P$, define
\begin{equation}
  S(P)
  :=
  \sum_{i\in\mathcal{I}}
  \sum_q P_{iq}s_{iq},
  ~~
  s_{iq}:=e_{iq}^\top\mH_{ii}e_{iq},
  \label{eq:app_self_objective_S}
\end{equation}
and
\begin{equation}
  X(P)
  :=
  2\sum_{(i,j)\in\mathcal{E}}
  \sum_{q,q'}P_{iq}P_{jq'}c_{ijqq'},
  ~~
  c_{ijqq'}:=e_{iq}^\top\mH_{ij}e_{jq'}.
  \label{eq:app_cross_objective_X}
\end{equation}
The cross-layer quadratic surrogate is
\begin{equation}
  C_{\rm quad}(P):=S(P)+X(P).
  \label{eq:app_C_quad}
\end{equation}
The factor $2$ in \eqref{eq:app_cross_objective_X} accounts for the symmetric $(i,j)$ and $(j,i)$ Hessian blocks. A common factor $1/2$ from the Taylor expansion is omitted because it does not affect the optimizer.
\end{definition}

Under the Kronecker self-approximation used in the main text,
\begin{equation}
  \mH_{ii}\approx \mA^{(i)}\otimes\mB^{(i)},
\end{equation}
and therefore
\begin{equation}
  s_{iq}=e_{iq}^\top\mH_{ii}e_{iq}
  \approx
  \tr\left(\mB^{(i)}\dW^{(i,q)}\mA^{(i)}(\dW^{(i,q)})^\top\right) =
  \ell_{iq}.
\end{equation}

This section presents a rigorous statement and proof of Proposition~\ref{prop:app_cross_never_worse}. 
We first formally restate Proposition~\ref{prop:app_cross_never_worse} below.
\begin{proposition}[Exact cross-aware optimization is never worse for the quadratic surrogate]
Let $\mathcal{F}$ be the feasible set defined by the MCKP assignment and budget constraints. Let
\begin{equation}
  P_{\rm self}\in\operatorname*{arg\,min}_{P\in\mathcal{F}}S(P), \quad
  P_{\rm cross}\in\operatorname*{arg\,min}_{P\in\mathcal{F}}C_{\rm quad}(P).
\end{equation}
Then
\begin{equation}
  C_{\rm quad}(P_{\rm cross})
  \le
  C_{\rm quad}(P_{\rm self}).
\end{equation}
If $P_{\rm self}$ is not a minimizer of $C_{\rm quad}$ over $\mathcal{F}$, then the inequality is strict.
\end{proposition}

\begin{proof}
By definition, $P_{\rm cross}$ is a global minimizer of $C_{\rm quad}$ over $\mathcal{F}$. Since $P_{\rm self}\in\mathcal{F}$,
\begin{equation}
  C_{\rm quad}(P_{\rm cross}) =
  \min_{P\in\mathcal{F}}C_{\rm quad}(P)
  \le
  C_{\rm quad}(P_{\rm self}).
\end{equation}
If $P_{\rm self}$ is not a minimizer of $C_{\rm quad}$, then its objective value is strictly larger than the minimum value, which gives strict inequality.
\end{proof}

\begin{proposition}[The self-only gap is controlled by omitted cross terms]
\label{prop:app_gap_controlled_by_cross}
With $P_{\rm self}$ and $P_{\rm cross}$ as in Proposition~\ref{prop:app_cross_never_worse},
\begin{equation}
  0
  \le
  C_{\rm quad}(P_{\rm self})-C_{\rm quad}(P_{\rm cross})
  \le
  X(P_{\rm self})-X(P_{\rm cross})
  \le
  2\max_{P\in\mathcal{F}}|X(P)|.
  \label{eq:app_self_gap_cross_terms}
\end{equation}
\end{proposition}

\begin{proof}
The first inequality follows from Proposition~\ref{prop:app_cross_never_worse}. For the second inequality, expand the difference:
\begin{align}
  C_{\rm quad}(P_{\rm self})-C_{\rm quad}(P_{\rm cross})
  &=
  \left[S(P_{\rm self})+X(P_{\rm self})\right]
  -
  \left[S(P_{\rm cross})+X(P_{\rm cross})\right] \\
  &=
  \left[S(P_{\rm self})-S(P_{\rm cross})\right]
  +
  \left[X(P_{\rm self})-X(P_{\rm cross})\right].
\end{align}
Since $P_{\rm self}$ minimizes $S$ over $\mathcal{F}$ and $P_{\rm cross}\in\mathcal{F}$,
\begin{equation}
  S(P_{\rm self})\le S(P_{\rm cross}).
\end{equation}
Therefore
\begin{equation}
  S(P_{\rm self})-S(P_{\rm cross})\le0,
\end{equation}
and hence
\begin{equation}
  C_{\rm quad}(P_{\rm self})-C_{\rm quad}(P_{\rm cross})
  \le
  X(P_{\rm self})-X(P_{\rm cross}).
\end{equation}
The last inequality follows from the triangle inequality:
\begin{align}
  X(P_{\rm self})-X(P_{\rm cross})
  &\le
  |X(P_{\rm self})|+|X(P_{\rm cross})| \\
  &\le
  2\max_{P\in\mathcal{F}}|X(P)|.
\end{align}
Combining the three inequalities proves the result.
\end{proof}

\begin{example}[Self-costs alone can be arbitrarily poor in relative value]
\label{ex:app_self_only_bad_relative}
For every $R>0$ and $\rho\in(0,1)$, there exists a two-module quadratic problem with a positive definite Hessian such that all feasible assignments have exactly the same self-cost, yet a self-only allocation can be worse than the cross-aware optimum by an additive gap $4\rho R$ and by a multiplicative ratio
\begin{equation}
  \frac{1+\rho}{1-\rho}.
\end{equation}
As $\rho\uparrow1$, this ratio diverges to $+\infty$.
\end{example}

\begin{proof}
Consider two adjacent scalar modules. Each module has two feasible perturbation choices, represented by signs $s,t\in\{-1,+1\}$. Define
\begin{equation}
  e_1(s)=\sqrt{R}s,
  ~~
  e_2(t)=\sqrt{R}t.
\end{equation}
Let the Hessian over the two scalar parameters be
\begin{equation}
  \mH=
  \begin{pmatrix}
    1 & \rho\\
    \rho & 1
  \end{pmatrix}.
\end{equation}
The eigenvalues of $\mH$ are $1+\rho$ and $1-\rho$. Since $\rho\in(0,1)$, both eigenvalues are strictly positive, so $\mH\succ0$.

For an assignment $(s,t)$, the full quadratic cost is
\begin{align}
  C_{\rm quad}(s,t)
  &=
  \begin{pmatrix}e_1(s)&e_2(t)\end{pmatrix}
  \begin{pmatrix}
    1 & \rho\\
    \rho & 1
  \end{pmatrix}
  \begin{pmatrix}e_1(s)\\ e_2(t)\end{pmatrix} \\
  &=
  e_1(s)^2+e_2(t)^2+2\rho e_1(s)e_2(t) \\
  &=
  R+R+2\rho Rst \\
  &=
  2R+2\rho Rst.
  \label{eq:app_two_module_cross_cost}
\end{align}
The self-only cost drops the cross term:
\begin{equation}
  S(s,t)=e_1(s)^2+e_2(t)^2=2R.
\end{equation}
Thus all four assignments have identical self-cost, and a self-only criterion has no information with which to distinguish the sign patterns. In particular, an assignment with $st=+1$ is self-optimal. Its full quadratic cost is
\begin{equation}
  C_{\rm same}=2R+2\rho R.
\end{equation}
The cross-aware optimum chooses $st=-1$, for which
\begin{equation}
  C_{\rm opp}=2R-2\rho R.
\end{equation}
The additive gap is
\begin{equation}
  C_{\rm same}-C_{\rm opp} =
  (2R+2\rho R)-(2R-2\rho R) =
  4\rho R.
\end{equation}
The multiplicative ratio is
\begin{equation}
  \frac{C_{\rm same}}{C_{\rm opp}} =
  \frac{2R+2\rho R}{2R-2\rho R} =
  \frac{1+\rho}{1-\rho}.
\end{equation}
Since $1-\rho\to0$ as $\rho\uparrow1$, the ratio diverges to $+\infty$.
\end{proof}

\begin{lemma}[PSD Hessian blocks imply a Cauchy--Schwarz cross-term bound]
\label{lem:app_psd_cauchy_bound}
Let
\begin{equation}
  \begin{pmatrix}
    \mH_{ii} & \mH_{ij}\\
    \mH_{ji} & \mH_{jj}
  \end{pmatrix}
  \succeq0,
  ~~
  \mH_{ji}=\mH_{ij}^\top.
\end{equation}
Then, for all compatible real vectors $x$ and $y$,
\begin{equation}
  |x^\top\mH_{ij}y|
  \le
  \sqrt{x^\top\mH_{ii}x}
  \sqrt{y^\top\mH_{jj}y}.
  \label{eq:app_psd_cauchy_bound}
\end{equation}
\end{lemma}

\begin{proof}
Define
\begin{equation}
  a:=x^\top\mH_{ii}x,
  ~~
  b:=y^\top\mH_{jj}y,
  ~~
  c:=x^\top\mH_{ij}y.
\end{equation}
Since the block matrix is positive semidefinite, for every $t\in\mathbb{R}$,
\begin{align}
  0
  &\le
  \begin{pmatrix}tx\\y\end{pmatrix}^\top
  \begin{pmatrix}
    \mH_{ii} & \mH_{ij}\\
    \mH_{ji} & \mH_{jj}
  \end{pmatrix}
  \begin{pmatrix}tx\\y\end{pmatrix} \\
  &=
  t^2x^\top\mH_{ii}x
  +t x^\top\mH_{ij}y
  +t y^\top\mH_{ji}x
  +y^\top\mH_{jj}y.
\end{align}
Because $\mH_{ji}=\mH_{ij}^\top$ and all quantities are real scalars,
\begin{equation}
  y^\top\mH_{ji}x =
  y^\top\mH_{ij}^\top x =
  x^\top\mH_{ij}y =
  c.
\end{equation}
Hence
\begin{equation}
  at^2+2ct+b\ge0
  ~~
  \text{for all }t\in\mathbb{R}.
  \label{eq:app_quadratic_nonnegative}
\end{equation}
If $a>0$, then the quadratic polynomial in \eqref{eq:app_quadratic_nonnegative} has nonpositive discriminant; otherwise it would take a negative value for some real $t$. Thus
\begin{equation}
  (2c)^2-4ab\le0,
\end{equation}
which is equivalent to $c^2\le ab$ and hence $|c|\le\sqrt{ab}$.

If $a=0$, then \eqref{eq:app_quadratic_nonnegative} becomes
\begin{equation}
  2ct+b\ge0
  ~~
  \text{for all }t\in\mathbb{R}.
\end{equation}
This is possible only if $c=0$; otherwise taking $t\to-\infty$ when $c>0$ or $t\to+\infty$ when $c<0$ gives a contradiction. Hence $|c|=0=\sqrt{ab}$ because $a=0$. The desired inequality follows in all cases.
\end{proof}

\begin{proposition}[A conservative cross-layer objective upper-bounds the quadratic surrogate]
\label{prop:app_robust_upper_bound}
For every interacting pair $(i,j)\in\mathcal{E}$ and candidates $(q,q')$, define
\begin{equation}
  s_{iq}:=e_{iq}^\top\mH_{ii}e_{iq},
  ~~
  s_{jq'}:=e_{jq'}^\top\mH_{jj}e_{jq'}.
\end{equation}
Assume $s_{iq}\ge0$ and $s_{jq'}\ge0$, and suppose there are constants $\rho_{ij}\in[0,1]$ such that
\begin{equation}
  |e_{iq}^\top\mH_{ij}e_{jq'}|
  \le
  \rho_{ij}\sqrt{s_{iq}s_{jq'}}
  ~~
  \text{for all }q,q'.
  \label{eq:app_rho_cross_bound}
\end{equation}
Define
\begin{equation}
  U_\rho(P)
  :=
  \sum_{i,q}P_{iq}s_{iq}
  +
  2\sum_{(i,j)\in\mathcal{E}}
  \sum_{q,q'}
  P_{iq}P_{jq'}\rho_{ij}\sqrt{s_{iq}s_{jq'}}.
  \label{eq:app_robust_upper_objective}
\end{equation}
Then every feasible assignment $P$ satisfies
\begin{equation}
  C_{\rm quad}(P)\le U_\rho(P).
  \label{eq:app_U_upper_bounds_C}
\end{equation}
If the Hessian block over $(i,j)$ is positive semidefinite, then \eqref{eq:app_rho_cross_bound} always holds with $\rho_{ij}=1$ by Lemma~\ref{lem:app_psd_cauchy_bound}.
\end{proposition}

\begin{proof}
Starting from Definition~\ref{def:app_cross_objectives},
\begin{equation}
  C_{\rm quad}(P) =
  \sum_{i,q}P_{iq}s_{iq}
  +
  2\sum_{(i,j)\in\mathcal{E}}
  \sum_{q,q'}P_{iq}P_{jq'}e_{iq}^\top\mH_{ij}e_{jq'}.
\end{equation}
For every cross term,
\begin{equation}
  e_{iq}^\top\mH_{ij}e_{jq'}
  \le
  |e_{iq}^\top\mH_{ij}e_{jq'}|
  \le
  \rho_{ij}\sqrt{s_{iq}s_{jq'}},
\end{equation}
where the second inequality is \eqref{eq:app_rho_cross_bound}. Since $P_{iq}P_{jq'}\ge0$, multiplying by $P_{iq}P_{jq'}$ preserves the inequality. Summing over all interacting pairs and candidates yields
\begin{equation}
  C_{\rm quad}(P)
  \le
  \sum_{i,q}P_{iq}s_{iq}
  +
  2\sum_{(i,j)\in\mathcal{E}}
  \sum_{q,q'}
  P_{iq}P_{jq'}\rho_{ij}\sqrt{s_{iq}s_{jq'}} =
  U_\rho(P).
\end{equation}
If the block Hessian for $(i,j)$ is positive semidefinite, then applying Lemma~\ref{lem:app_psd_cauchy_bound} with $x=e_{iq}$ and $y=e_{jq'}$ gives
\begin{equation}
  |e_{iq}^\top\mH_{ij}e_{jq'}|
  \le
  \sqrt{e_{iq}^\top\mH_{ii}e_{iq}}
  \sqrt{e_{jq'}^\top\mH_{jj}e_{jq'}} =
  \sqrt{s_{iq}s_{jq'}},
\end{equation}
which is \eqref{eq:app_rho_cross_bound} with $\rho_{ij}=1$.
\end{proof}

\begin{corollary}[Robust cross-aware allocation improves the conservative upper bound]
\label{cor:app_robust_never_worse_upper}
Let
\begin{equation}
  P_\rho\in\operatorname*{arg\,min}_{P\in\mathcal{F}}U_\rho(P).
\end{equation}
Then
\begin{equation}
  U_\rho(P_\rho)
  \le
  U_\rho(P_{\rm self}).
\end{equation}
\end{corollary}

\begin{proof}
This follows immediately because $P_\rho$ minimizes $U_\rho$ over the same feasible set $\mathcal{F}$ and $P_{\rm self}\in\mathcal{F}$.
\end{proof}


\end{document}